\documentclass{article}

\usepackage[final]{neurips_2022}

\usepackage[utf8]{inputenc} 
\usepackage[T1]{fontenc}    
\usepackage{hyperref}       
\usepackage{url}            
\usepackage{booktabs}       
\usepackage{amsfonts}       
\usepackage{nicefrac}       
\usepackage{microtype}      
\usepackage{xcolor}         
\usepackage{authblk}
\usepackage{amsmath}
\usepackage{amsthm}
\usepackage{wrapfig}
\usepackage{hyperref}
\usepackage{algpseudocode}
\usepackage{url}
\usepackage{graphicx}
\usepackage{cancel}
\usepackage[rightcaption]{sidecap}
\usepackage{color,soul}

\newtheorem{theorem}{Theorem} 
\newtheorem{theorem*}{Theorem}
\newtheorem{corollary}{Corollary}
\usepackage{color, colortbl}
\definecolor{Gray}{gray}{0.8}
\definecolor{LightCyan}{rgb}{0.88,1,1}
\definecolor{green(html/cssgreen)}{rgb}{0.0, 0.5, 0.0}

\hypersetup{
    colorlinks=true,
    linkcolor=black,
    citecolor=black,      
    urlcolor=blue
}
\usepackage{caption} 
\usepackage{array} 

\newcommand{\bx}{\mathbf{x}}

\newcommand{\bP}{\mathbf{P}}
\newcommand{\bM}{\mathbf{M}}
\newcommand{\bD}{\mathbf{D}}
\newcommand{\bW}{\mathbf{W}}
\newcommand{\bR}{\mathbf{R}}
\title{RNNs of RNNs: \\ Recursive Construction of Stable Assemblies of Recurrent Neural Networks}

\author[1,*]{\textbf{Leo Kozachkov}}
\author[2,*]{\textbf{Michaela Ennis}}
\author[1,3,4]{\textbf{Jean-Jacques Slotine}}

\affil[1]{Department of Brain and Cognitive Sciences, Massachusetts Institute of Technology}
\affil[2]{Division of Medical Sciences, Harvard University}
\affil[3]{Department of Mechanical Engineering, Massachusetts Institute of Technology}
\affil[4]{Google AI}
\affil[*]{Equal contribution}
\affil[ ]{\texttt{\{leokoz8,mennis,jjs\}@mit.edu}}

\begin{document}

\maketitle

\begin{abstract}
Recurrent neural networks (RNNs) are widely used throughout neuroscience as models of local neural activity. Many properties of single RNNs are well characterized theoretically, but experimental neuroscience has moved in the direction of studying multiple interacting \textit{areas}, and RNN theory needs to be likewise extended. We take a constructive approach towards this problem, leveraging tools from nonlinear control theory and machine learning to characterize when combinations of stable RNNs will themselves be stable. Importantly, we derive conditions which allow for massive feedback connections between interacting RNNs. We parameterize these conditions for easy optimization using gradient-based techniques, and show that stability-constrained `networks of networks' can perform well on challenging sequential-processing benchmark tasks. Altogether, our results provide a principled approach towards understanding distributed, modular function in the brain.
\end{abstract}

\addtocontents{toc}{\protect\setcounter{tocdepth}{0}}

\section{Introduction}
The combination and reuse of primitive ``modules" has enabled a great deal of progress in computer science, engineering, and biology. Modularity is particularly apparent in the structure of the brain, as different parts are specialized for different functions \citep{kandel2000principles}. Accordingly, most experimental studies throughout the history of neuroscience have focused on a single brain area in association with a single behavior \citep{Abbott_Svoboda_2020}. Similarly, RNN models of brain function have mostly been limited to a single RNN modeling a single area. However, neuroscience is entering an age where recording from many different brain areas simultaneously during complex behaviors is possible. As experimental neuroscience has shifted towards multi-area recordings, computational techniques for analyzing, modeling, and interpreting these multi-area recordings have blossomed \citep{michaels2020goal,mashour2020conscious,Abbott_Svoboda_2020, Perich_2021,semedo2019cortical,yang2021towards,Machado_Kauvar_Deisseroth_2022}. Despite this, RNN theory has lagged behind. 

The theoretical question of RNN stability is crucial for understanding information propagation and manipulation \citep{vogt2020lyapunov,engelken2020lyapunov,kozachkov2022robust}. The conditions under which single, autonomous RNNs are chaotic or stable are well-studied, in particular when the RNN weights are randomly chosen and the number of neurons tends to infinity \citep{sompolinsky1988chaos,engelken2020lyapunov}. However, there is very little work addressing the theoretical question of stability in `networks of networks'. Two facts make this question challenging. Firstly, connecting two stable systems does not, in general, lead to a stable overall system. This is true even for linear systems. Secondly, there is a massive amount of feedback between brain areas, so one cannot reasonably assume near-decomposability \citep{Simon_1962,Abbott_Svoboda_2020}. 

Given that the brain seems to dynamically reorganize and adapt interareal connectivity to meet task demands and environmental constraints \citep{miller2001integrative,sych2022dynamic}, this question of how stability is maintained is of the utmost importance. Here we take a bottom-up approach, more specifically asking ``what stability properties of the individual modules lend themselves to rapid reorganization?"

\subsection{Contraction Analysis}\label{subsection:contraction-intro}
We focus on a special type of stability, known as contractive stability \citep{lohmiller1998contraction}. Loosely, a contracting system is a dynamical system that forgets its initial conditions exponentially quickly. Contractive stability is a strong form of dynamical stability which implies many other forms of stability, such as certain types of input-to-state stability \citep{sontag2010contractive}. See Section \ref{Appendix:Supplementary Math:supp_math} for a mathematical primer on contraction analysis. 

Contraction analysis has found wide application in nonlinear control theory \citep{manchester2017control}, synchronization \citep{pham2007stable}, and robotics \citep{chung2009cooperative}, but has only recently begun to find application in neuroscience and machine learning \citep{boffi2020learning,wensing2020beyond,kozachkov2020achieving,revay2020contracting,jafarpour2021robust,kozachkov2022robust, centorrino2022contraction,burghi2022distributed,kozachkov2022generalization}. Contraction analysis is useful for neuroscience because it is directly applicable to systems with external inputs. It also allows for modular stability-preserving \emph{combination} properties to be derived (Figure \ref{fig: Cartoon}). The resulting contracting combinations can involve individual systems with different dynamics, as long as those dynamics are contracting \citep{slotine2001combos,modular}.

Moreover, modular stability and specifically contractive stability have relevance to evolutionary biology~\citep{Simon_1962, slotine2001combos}. In particular, it is thought that the majority of traits that have developed over the last 400+ million years are the result of evolutionary forces acting on regulatory elements that combine core components, rather than mutations in the core components themselves. This mechanism of action makes meaningful variation in population phenotypes much more feasible to achieve, and is appropriately titled ``facilitated variation" \citep{gerhart2007fv}. In addition to the biological evidence for facilitated variation, computational models have demonstrated that this approach produces populations which are better able to generalize to new environments \citep{parter2008fv}, an ability that will be critical to further develop in deep learning systems. However, the tractability of these evolutionary processes hinges on some mechanism for ensuring stability of combinations. Because contraction analysis tools allow complicated contracting systems to be built up recursively from simpler elements, this form of stability would be well suited for biological systems \citep{slotine2012links}. Our work with the Sparse Combo Net in Section \ref{section:experiments-intro} has direct parallels to facilitated variation, in that we train this combination network architecture \textit{only} through training connections between contracting subnetworks.  

\begin{figure}[h]
\centering
\includegraphics[width=0.95\textwidth,keepaspectratio]{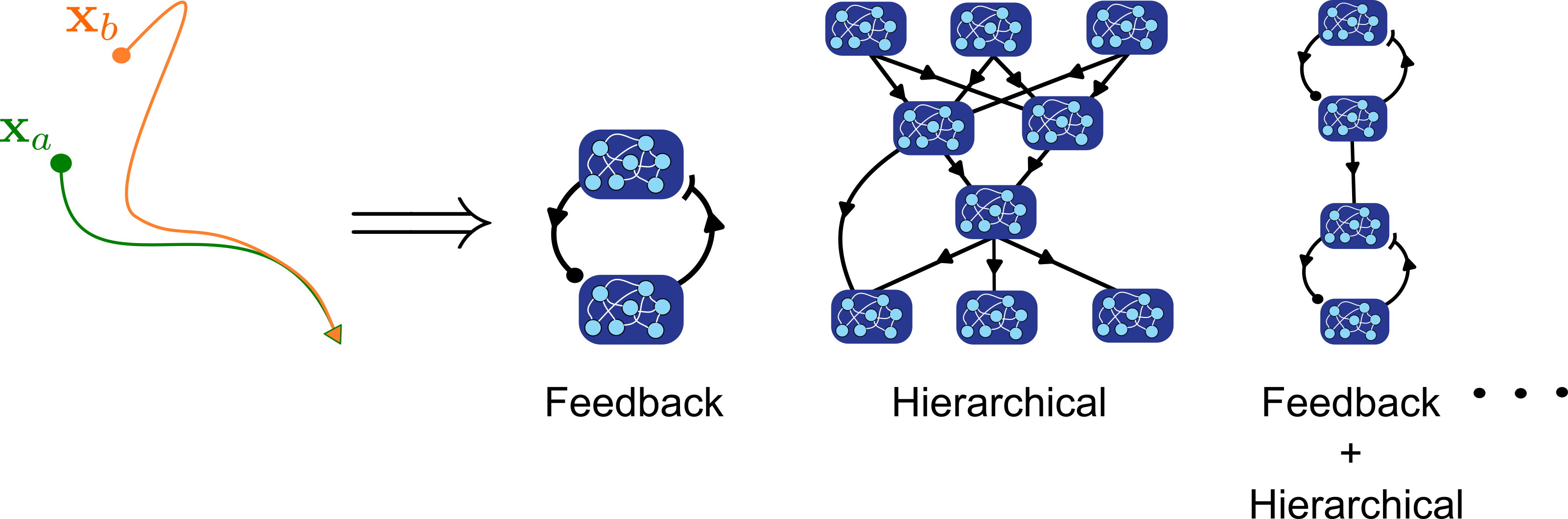}
\caption{Contractive stability implies a modularity principle. Because contraction analysis tools allow complicated contracting systems to be built up recursively from simpler elements, this form of stability is well suited for understanding biological systems. Contracting combinations can be made between systems with very different dynamics, as long as those dynamics are contracting.}
\label{fig: Cartoon}
\end{figure}

Ultimately, our contributions are three-fold: 
\begin{itemize}
    \setlength\itemsep{.1em}
    \item A novel parameterization for feedback combination of contracting RNNs that enables direct optimization using standard deep learning libraries.
    \item  Novel contraction conditions for continuous-time nonlinear RNNs, to use in conjunction with the combination condition. We also identify flaws in stability proofs from prior literature. 
    \item Experiments demonstrating that our `network of networks' sets a new state of the art for stability-constrained RNNs on benchmark sequential processing tasks. 
\end{itemize}

\section{Network of Networks Model}\label{section:model-math}
In this paper we analyze rate-based neural networks. Unlike spiking neural networks, these models are continuous and smooth. We consider the following RNN introduced by \cite{Wilson_Cowan_1972}, which may be viewed as an approximation to a more biophysically-detailed spiking network: 
\begin{equation}\label{eq:RNN}
\begin{split}
\tau \dot{\mathbf{x}} = -\mathbf{x} + \mathbf{W}\phi(\mathbf{x}) + \mathbf{u}(t)
\end{split}
\end{equation}
Here $\tau > 0$ is the time-constant of the network \citep{dayan2005theoretical}, and the vector $\mathbf{x} \in \mathbb{R}^n$ contains the voltages of all $n$ neurons in the network. The voltages are converted into firing-rates through a static nonlinearity $\phi$. We only consider monotonic activation functions with bounded slope: in other words, $0 \leq \phi' \leq g$ (unless otherwise noted). We do not restrain the sign of the nonlinearity. Common example nonlinearities $\phi(x)$ that satisfy these constraints are hyperbolic tangent and ReLU. The matrix $\mathbf{W} \in \mathbb{R}^{n \times n}$ contains the synaptic weights of the RNN. It is this matrix that ultimately determines the stability of \eqref{eq:RNN}, and will be a main target for our analysis. Finally, $\mathbf{u}(t)$ is the potentially time-varying external input into the network, capturing both explicit input into the RNN from the external world, as well as unmodeled dynamics from other brain areas. In the context of training, the time-varying inputs $\mathbf{u}(t)$ come from the task. For example, if the task is sequential image recognition, then the input will be a sequence of pixel intensities. The output of our network similarly depends on the task. In classification settings, one may choose to use a softmax output layer; in regression, one may choose a linear readout. 

Note that \eqref{eq:RNN} is equivalent to another commonly used class of RNNs where the term $\mathbf{W}\mathbf{x} + \mathbf{u}$ appears inside the nonlinearity. See Section \ref{Appendix:Supplementary Math:two_different_rnns} or \citep{miller2012mathematical} for details. Our mathematical results apply equally well to both types of RNNs.

In order to extend \eqref{eq:RNN} into a model of multiple interacting neural populations, we introduce the index $i$, which runs from $1$ to $p$, where $p$ is the total number of RNNs in the collection of RNNs. For now we will assume linear interactions between RNNs, because this is the simplest case. The linearity assumption can also be motivated by the fact that RNNs have been found to be well-approximated by linear systems in many neuroscience contexts \citep{sussillo2013opening,langdon2022latent}. This leads to the following equation for the `network of networks':
\begin{equation}\label{eq:combo_RNN}
\begin{split}
\tau \dot{\mathbf{x}}_i = -\mathbf{x}_i + \mathbf{W}_i\phi(\mathbf{x}_i) + \sum_{j=1}^{p} \mathbf{\mathbf{L}}_{ij} \mathbf{x}_j + \mathbf{u}_i(t) \hspace{.5cm}  \forall i = 1 \cdots p
\end{split}
\end{equation}
where the matrix $\mathbf{L}_{ij}$ captures the interaction between RNN $i$ and RNN $j$. If RNN $i$ has $N_i$ neurons and RNN $j$ has $N_j$ neurons, then $\mathbf{L}_{ij}$ is an $N_i \times N_j$ matrix. 

We can now formalize the question posed in the introduction. Namely, "what types of connections between stable RNNs automatically preserve stability?" becomes "what restrictions on $\mathbf{\mathbf{L}}_{ij}$ must be met in order to ensure overall stability of the network of networks?". We will now derive two combination `primitives', negative feedback and hierarchical, which allow for recursive connection of contracting modules. 

\subsection{Generalized Negative Feedback Between RNNs Preserves Stability}
We set aside for a moment the problem of determining when \eqref{eq:RNN} is contracting. For now, assume that we have a collection of $p$ contracting RNNs interacting through equation \eqref{eq:combo_RNN}. Recall that contraction is defined with respect to a \textit{metric}, a way of measuring distances between trajectories in state space. Thus, the $i$th RNN is contracting with respect to some metric $\mathbf{M}_i$. We assume for simplicity that this metric is constant, which means that $\mathbf{M}_i$ is simply a symmetric, positive definite matrix. In the case where every RNN receives feedback from every other, we can preserve stability by ensuring these connections are negative feedback. In the simplest case where all RNN modules are contracting in the identity metric, the negative feedback may be written as:
\[\mathbf{L}_{ij} = -\mathbf{L}^T_{ji} \]
This is a well known result from the contraction analysis literature \citep{modular}. Our first novel contribution is a generalization and parameterization of this negative feedback which allows for direct optimization using gradient-based techniques. In particular, if each $\mathbf{L}_{ij}$ is parameterized as follows:
\begin{equation}\label{eq:feedback_combo}
\begin{split}
\mathbf{L}_{ij} = \mathbf{B}_{ij} - \mathbf{M}^{-1}_i\mathbf{B}^T_{ji}\mathbf{M}_j \hspace{.5cm} \forall i,j
\end{split}
\end{equation}
for arbitrary matrix $\mathbf{B}_{ij}$, then the overall network of networks retains the assumed contraction properties of the RNN subnetworks. We provide a detailed proof in Section \ref{subsection:feedback_combo_proof}, but the basic idea relies on ensuring  skew-symmetry of $\mathbf{L}_{ij}$ \textit{in the appropriate metric}. This can be achieved via the constraint:
\[\mathbf{M}_i\mathbf{L}_{ij} = -\mathbf{L}^T_{ji}\mathbf{M}_j \]
Plugging \eqref{eq:feedback_combo} into the above expression verifies that it is indeed satisfied. Because contraction analysis relies on analyzing the symmetric part of Jacobian matrices, the skew-symmetry of $\mathbf{L}_{ij}$ `cancels out' when  computing the symmetric part, and leaves the stability of the subnetwork RNNs untouched. In the remainder of the paper, in the experimental sections, we will use feedback combinations constrained by \eqref{eq:feedback_combo}. However, it is possible to significantly generalize this condition (see \ref{subsection:feedback_combo_proof} for details). It is also possible to replace the linear feedback between modules with saturating nonlinearities, by using the "small-gain" results in \citet{slotine2001combos} and \citet{tabareau2006notes}. In the remainder of the paper, we optimize the $\mathbf{B}_{ij}$ matrices directly using backpropagation and the Adam optimizer implemented in PyTorch \citep{kingma2014adam,paszke2019pytorch}. 

\paragraph{Recursive Properties of Contracting Combinations}
The feedback combination \eqref{eq:feedback_combo}, taken together with hierarchical combinations, may be used as combination primitives for recursively constructing complicated networks of networks while automatically maintaining stability.  The recursion comes from the fact that once a modular system is shown to be contracting it may be treated as a single contracting system, which may in turn be combined with other contracting systems, \textit{ad infinitum}. Note that while we require linear feedback connections \eqref{eq:feedback_combo}, hierarchical interareal connections may be nonlinear \citep{lohmiller1998contraction}.

\section{Various Ways to Achieve Local
Contraction}\label{section:single-rnn}
In this section we return to the question of achieving contraction in the subnetwork RNNs. Recall that we wish to find restrictions on $\mathbf{W}_i$ such that the $i$th subnetwork RNN, described by \eqref{eq:RNN}, is contracting. Here we derive five such novel conditions (see Section \ref{Appendix:Proofs} for detailed proofs). As we will discuss, not all contraction conditions are equally useful - for example some conditions are easier to optimize or have higher model capacity than others. In this section we also point out some flaws in existing stability proofs in the RNN literature, and suggest some pathways towards correcting them. \\

\vspace{0.1cm}
\begin{theorem}[Absolute Value Restricted Weights]
\label{theorem: absolutevaluetheorem}
Let $|\mathbf{W}|$ denote the matrix formed by taking the element-wise absolute value of $\mathbf{W}$.  If there exists a positive, diagonal $\mathbf{P}$ such that:
\[\mathbf{P}(g|\mathbf{W}|-\mathbf{I}) +(g|\mathbf{W}|-\mathbf{I})^T\mathbf{P} \prec 0 \]
with $g$ being the maximum slope of $\phi$, then \eqref{eq:RNN} is contracting in metric $\mathbf{P}$. If $W_{ii} \leq 0$, then $|W|_{ii}$ may be set to zero to reduce conservatism.
\end{theorem}
It is easy to find matrices that satisfy Theorem \ref{theorem: absolutevaluetheorem}, and given a matrix the condition is as easy to check as linear stability is. Moreover, the condition guarantees we can obtain a metric that the system is known to contract in (see Section \ref{section:init-info} for details). It is less straightforward to enforce this condition during training, however we found that subnetworks constrained by Theorem \ref{theorem: absolutevaluetheorem} can achieve high performance in practice by simply fixing $\mathbf{W}$ and only optimizing the connections \textit{between} subnetworks (Section \ref{section:main-results}). As there are fewer parameters to optimize, this training technique is faster. 

\vspace{0.2cm}
\begin{theorem}[Symmetric Weights]
\label{theorem: symmetricweightstheorem}
If $ \ \mathbf{W} = \mathbf{W}^T$ and $  \ g\mathbf{W} \prec \mathbf{I}$, and $ \ \phi' > 0$, then (\ref{eq:RNN}) is contracting.
\end{theorem}

It has been known since the early 1990s that if \eqref{eq:RNN} is autonomous (i.e the input $\mathbf{u}$ is constant) and has symmetric weights with eigenvalues less than $1/g$, then there exists a unique fixed point that the network converges to from any initial condition \citep{matsuoka1992stability}. Theorem \ref{theorem: symmetricweightstheorem} generalizes this statement to say that if \eqref{eq:RNN} has symmetric weights with eigenvalues less than $1/g$, it is contracting. This includes previous results as a special case, because an \textit{autonomous} contracting system has a unique fixed point which the network converges to from any initial condition. 

\vspace{0.2cm}
\begin{theorem}[Product of Diagonal and Orthogonal Weights]
\label{theorem: PQPtheorem}
If there exists positive diagonal matrices $\mathbf{P}_1$ and $\mathbf{P}_2$, as well as $\mathbf{Q} = \mathbf{Q} ^T \succ 0$ such that
\[ \mathbf{W} = -\mathbf{P}_1 \mathbf{Q} \mathbf{P}_2 \]
then (\ref{eq:RNN}) is contracting in metric $\mathbf{M} = (\mathbf{P}_1 \mathbf{Q} \mathbf{P}_1)^{-1}$. 
\end{theorem}

In contrast to the first two contraction conditions, Theorem \ref{theorem: PQPtheorem} is very easy to optimize. To meet the constraint that the $\mathbf{P}$ matrices are positive, one can parameterize their diagonal elements as $P_{ii} = d^2_{i} + \epsilon$, for some small positive constant $\epsilon$, and optimize $d_i$ directly. To meet the positive definiteness constraint on $\mathbf{Q}$, one may parameterize it as $\mathbf{Q} = \mathbf{E}^T\mathbf{E} + \epsilon\mathbf{I}$ and optimize $\mathbf{E}$ directly.

\vspace{0.2cm}
\begin{theorem}[Triangular Weights]
\label{theorem: triangularweightstheorem}
If $g\mathbf{W}-\mathbf{I}$ is triangular and Hurwitz, then (\ref{eq:RNN}) is contracting in a diagonal metric.  
\end{theorem}

Theorem \ref{theorem: triangularweightstheorem} follows from the fact that a hierarchy of contracting systems is also contracting.

\vspace{0.2cm}
\begin{theorem}[Singular Value Restricted Weights]
\label{theorem: singularvaluetheorem}
If there exists a positive diagonal matrix $\mathbf{P}$ such that:
\[g^2\mathbf{W}^T\mathbf{P}\mathbf{W} - \mathbf{P} \prec 0 \]
\noindent then (\ref{eq:RNN}) is contracting in metric $\mathbf{P}$.  
\end{theorem}

In the case of discrete-time RNNs, this contraction condition has been proved by many different authors in many different settings. When $\mathbf{P} = \mathbf{I}$, it is known as the echo-state condition for discrete-time RNNs \citep{jaeger2001echo}. This was then generalized to diagonal $\mathbf{P}$ by \cite{buehner2006tighter}. More recently, the original echo-state condition was rediscovered by \cite{miller2018stable} in the machine learning literature. Following this rediscovery, the condition was generalized to $\mathbf{P} \neq \mathbf{I}$ by \cite{revay2020contracting}. Here we show that it applies to continuous-time RNNs as well. 

\subsection{What do the Jacobian Eigenvalues Tell Us?}\label{section:stability-in-ml}
Several recent papers in ML, e.g  \citep{haber2017stable,chang2019antisymmetricrnn}, claim that a sufficient condition for stability of the nonlinear system:
\[\dot{\mathbf{x}} = \mathbf{f}(\mathbf{x},t)\]
is that the associated Jacobian matrix $\mathbf{J}(\mathbf{x},t) = \frac{\partial \mathbf{f}}{\partial \mathbf{x}}$ has eigenvalues whose real parts are strictly negative, i.e:
\[\max_i \text{Re}(\lambda_i(\mathbf{J}(\mathbf{x},t)) \leq -\alpha\]
with $\alpha>0$. However, this claim is generally false - see Section 4.4.2 in \citep{slotine1991applied}. 

In the \textit{specific} case of the RNN \eqref{eq:RNN}, it appears that the eigenvalues of the symmetric part of $\mathbf{W}$ do provide information on global stability in a number of applications. For example, in \citep{matsuoka1992stability} it was shown that if $\mathbf{W}_s = \frac{1}{2}(\mathbf{W} + \mathbf{W}^T)$ has all its eigenvalues less than unity, and $\mathbf{u}$ is constant, then \eqref{eq:RNN} has a unique, globally asymptotically stable fixed point. This condition also implies that the real parts of the eigenvalues of the Jacobian are uniformly negative. Moreover, in \citep{chang2019antisymmetricrnn} it was shown that setting the symmetric part of $\mathbf{W}_s = \frac{1}{2}(\mathbf{W} + \mathbf{W}^T)$ almost equal to zero (yet slightly negative) led to rotational, yet stable dynamics in practice. This leads us to the following theorem, which shows that if the slopes of the activation functions change sufficiently slowly as a function of time, then the condition in \citep{matsuoka1992stability} in fact implies global contraction of (\ref{eq:RNN}).

\vspace{0.2cm}
\begin{theorem}\label{theorem: Wdiagstabtheorem}
Let $\mathbf{D}$ be a positive, diagonal matrix with $D_{ii} = \frac{d\phi_i}{dx_i}$, and let $\mathbf{P}$ be an arbitrary, positive diagonal matrix. If:

\[ (g\mathbf{W}-\mathbf{I})\mathbf{P} + \mathbf{P}(g\mathbf{W}^T-\mathbf{I}) \preceq -c\mathbf{P} \hspace{.5cm} \text{and} \hspace{.5cm} \dot{\mathbf{D}} - cg^{-1}\mathbf{D} \preceq -\beta\mathbf{D}\]
for $c,\beta > 0$, then (\ref{eq:RNN}) is contracting in metric $\mathbf{D}$ with rate $\beta$. 
\end{theorem}

We stress however, that it is an open question whether or not diagonal stability of $\mathbf{W}$ implies that \eqref{eq:RNN} is contracting. It has been conjectured that diagonal stability of $g\mathbf{W}-\mathbf{I}$ is a sufficient condition for global contraction of \eqref{eq:RNN} \citep{revay2020lipschitz}, however this has been difficult to prove. To better characterize this conjecture, we present Theorem \ref{theorem: Wdiagstabcounterexampletheorem}, which shows by way of counterexample that diagonal stability of $g\mathbf{W}-\mathbf{I}$ does not imply global contraction in a \textit{constant} metric for (\ref{eq:RNN}).  
\vspace{0.2cm}
\begin{theorem}
\label{theorem: Wdiagstabcounterexampletheorem}
Satisfaction of the condition \hspace{1mm} $g\mathbf{W}_{sym} - \mathbf{I} \prec 0$ \hspace{1mm} is \textbf{not} sufficient to show global contraction of the general nonlinear RNN (\ref{eq:RNN}) in any \textbf{constant} metric. High levels of antisymmetry in $\mathbf{W}$ can make it impossible to find such a metric, which we demonstrate via a $2 \times 2$ counterexample of the following form, with $c \geq 2$ when $g=1$: $\mathbf{W} = \begin{bmatrix} 
    0 & -c \\
    c & 0 
  \end{bmatrix}$
\end{theorem}

\section{Stability-Constrained Network of Networks Perform Well on Benchmarks}\label{section:experiments-intro}
A natural concern is that stability of an RNN may come at the cost of its expressivity, which is particularly relevant for integrating information over long timescales. To investigate whether this might be an issue for our model, we trained a stability-constrained network-of-networks on three benchmark sequential image classification tasks: sequential MNIST, permuted seqMNIST, and sequential CIFAR10. These tasks are often used to measure information processing ability over long sequences \citep{le2015simple}. Images are presented pixel-by-pixel, and the network makes a prediction at the end of the sequence. In permuted seqMNIST, pixels are input in a fixed but random order. 

All of our experiments were done on networks governed by \eqref{eq:combo_RNN}. The nonlinear subnetwork RNNs were connected to each other via linear all-to-all negative feedback, given by \eqref{eq:feedback_combo}. For all subnetworks we use the ReLU activation function. To enforce contraction of each individual subnetwork, we focused on two stability constraints from our theoretical results: Theorems \ref{theorem: absolutevaluetheorem} and \ref{theorem: singularvaluetheorem}. In the case of Theorem \ref{theorem: absolutevaluetheorem}, we did not train the individual subnetworks' weight matrices, but only trained the connections \textit{between} subnetworks (Figure \ref{fig:network-cartoon}B). For Theorem \ref{theorem: singularvaluetheorem}, we trained all parameters of the model (Figure \ref{fig:network-cartoon}C). 

We refer to networks with subnetworks constrained by Theorem \ref{theorem: absolutevaluetheorem} as `Sparse Combo Nets' and to networks with subnetworks constrained by Theorem \ref{theorem: singularvaluetheorem} as `SVD Combo Nets'. Throughout the experimental results we use the notation `$p \times n$ network' - such a network consists of $p$ distinct subnetwork RNNs, with each such subnetwork RNN containing $n$ units. 

\begin{figure}[h]
\centering
\includegraphics[width=\textwidth,keepaspectratio]{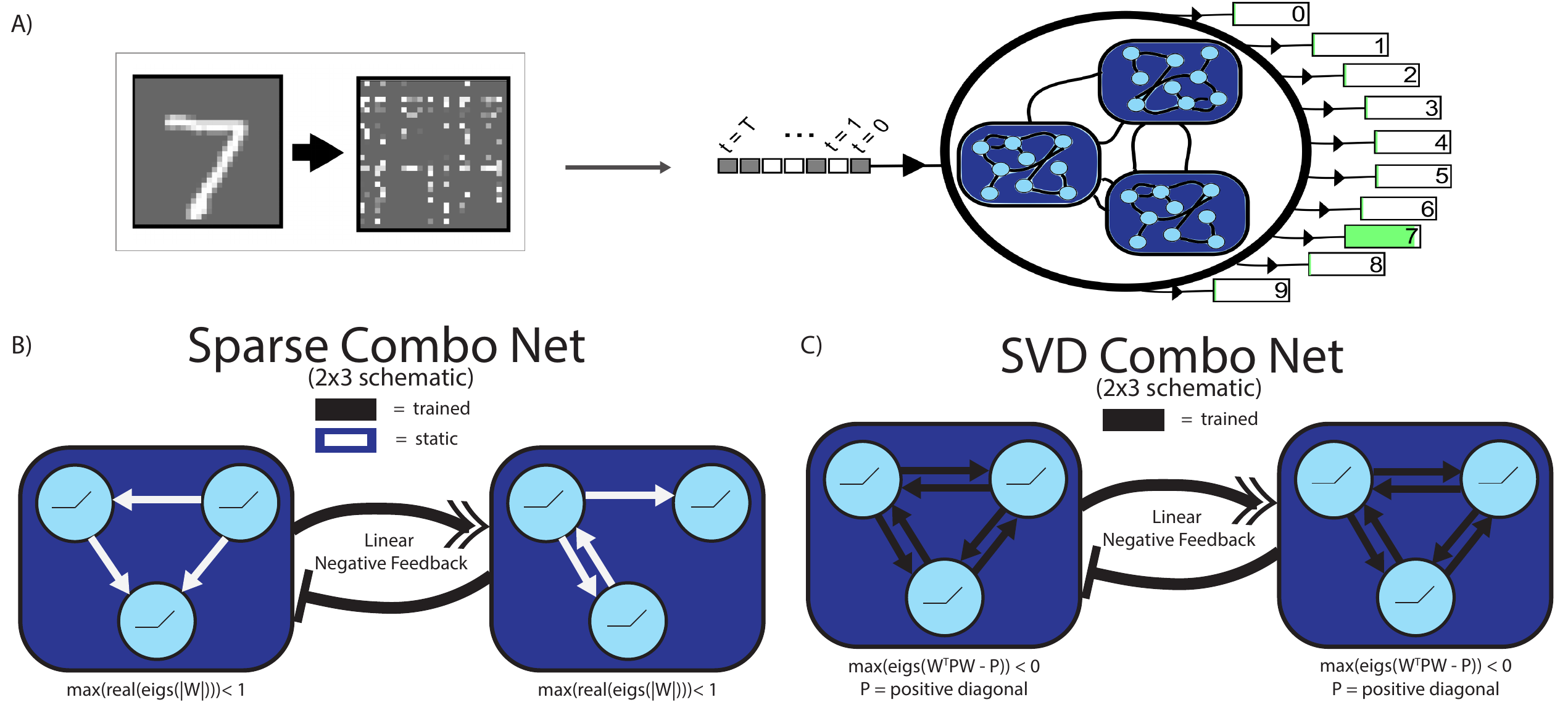}
\caption{Summary of task structure and network architectures. Images from MNIST (or CIFAR10) were flattened into an array of pixels and fed sequentially into the modular `network of networks', with classification based on the output at the last time-step. For MNIST, each image was also permuted in a fixed manner (A). The subnetwork `modules' of our architecture were constrained to meet either Theorem \ref{theorem: absolutevaluetheorem} via sparse initialization (B) or Theorem \ref{theorem: singularvaluetheorem} via direct parameterization (C). Linear negative feedback connections were trained between the subnetworks according to \eqref{eq:feedback_combo}.}
\label{fig:network-cartoon}
\end{figure}

\subsection{Network Initialization and Training}\label{section:init-info}
For the Sparse Combo Net we were not able to find a parameterization to continuously update the internal RNN weights during training in a way that preserved contraction. However, it is easy to randomly generate matrices with a particular likelihood of meeting the Theorem \ref{theorem: absolutevaluetheorem} condition by selecting an appropriate sparsity level and limit on entry magnitude. Sparsity in particular is of interest due to its relevance in neurobiology and machine learning, so it is convenient that the condition makes it easy to verify stability of many different sparse RNNs. As $g=1$ for ReLU activation, we check potential subnetwork matrices $\mathbf{W}$ by simply verifying linear stability of $|\mathbf{W}| - \mathbf{I}$. 

Because every RNN meeting the condition has a corresponding well-defined stable LTI system contracting in the same metric, it is also easy to find a metric to use in our training algorithm: solving for $\mathbf{M}$ in $-\mathbf{I} = \mathbf{M}\mathbf{A} + \mathbf{A}^{T}\mathbf{M}$ will produce a valid metric for any stable LTI system $\mathbf{A}$ \citep{slotine1991applied}. We utilize the fact that Hurwitz Metzler matrices are diagonally stable to improve efficiency of computing $\mathbf{M}$ (as well as in our proof of Theorem \ref{theorem: absolutevaluetheorem}). 

  We therefore randomly generated fixed subnetworks satisfying Theorem \ref{theorem: absolutevaluetheorem} and trained only the linear connections between them (Figure \ref{figure:example-training}), as well as the linear input and output layers. More information on network initialization, hyperparameter tuning, and training algorithm is provided in Section \ref{Appendix:Experiments}.

\begin{figure}[h]
\centering
\includegraphics[width=\textwidth,keepaspectratio]{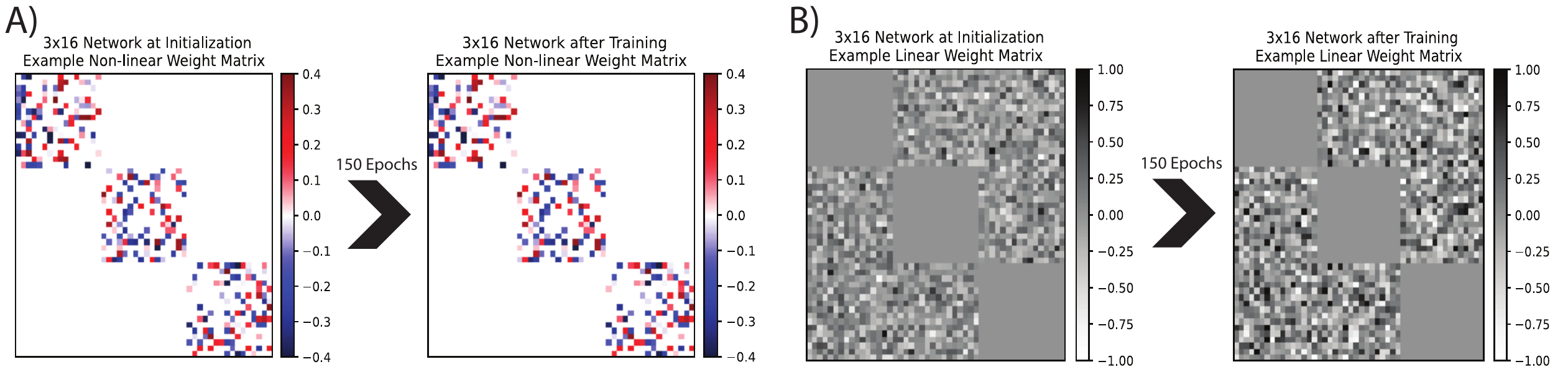}
\caption{Example $3$x$16$ Sparse Combo Net. Nonlinear intra-subnetwork weights are initialized using a set sparsity, and do not change in training (A). Linear inter-subnetwork connections are constrained to be antisymmetric with respect to the overall network metric, and are updated in training (B).}
\label{figure:example-training}
\end{figure}

For the SVD Combo Net on the other hand, we ensured contraction of each subnetwork RNN by direct parameterization (described in Section \ref{Appendix:SVDNet}), thus allowing all weights to be trained. 

\subsection{Results}\label{section:main-results}
The Sparse Combo Net architecture achieved the highest overall performance on both permuted seqMNIST and seqCIFAR10, with 96.94\% and 65.72\% best test accuracies respectively - thereby setting a new SOTA for stable RNNs (Table \ref{table:sota}). Furthermore, we were able to reproduce SOTA scores over several repetitions, including 10 trials of seqCIFAR10. Along with repeatability of results, we also show that the contraction constraint on the connections between subnetworks ($\mathbf{L}$ in \eqref{eq:feedback_combo}) is important for performance, particularly in the Sparse Combo Net (Section \ref{section:repeats}). 

Additionally, we profile how various architecture settings impact performance of our networks. In both networks, we found that increasing the total number of neurons improved task performance, but with diminishing returns (Section \ref{section: size}). We also found that the sparsity of the hidden-to-hidden weights in Sparse Combo Net had a large impact on the final network performance (Section \ref{section: sparse}).

\begin{table}[h!]
\small
\centering
\begin{tabular}{ | m{1.75cm} || m{0.7cm} | m{0.8cm} || m{1.3cm} | m{1.3cm} | m{1.3cm} || m{0.85cm} | m{0.85cm} | m{0.85cm} | }
\hline
 Name & Stable RNN? & Params & sMNIST \newline Repeats \newline Mean (n) \newline [Min] & psMNIST \newline Repeats \newline Mean (n) \newline [Min] & sCIFAR10 \newline Repeats \newline Mean (n) \newline [Min] & Seq \newline MNIST \newline Best & PerSeq \newline MNIST \newline Best & Seq \newline CIFAR \newline Best \\
\hline\hline
LSTM \citep{chang2019antisymmetricrnn} & & 68K & \centering --- & \centering --- & \centering --- & 97.3\% & 92.7\% & 59.7\% \\
\hline
Transformer \newline \citep{trinh2018cifar} & & 0.5M & \centering --- & \centering --- & \centering --- & 98.9\% & 97.9\% & 62.2\% \\
\hline\hline
Antisymmetric \newline \citep{chang2019antisymmetricrnn} & ? & 36K & \centering --- & \centering --- & \centering --- & 98\% & 95.8\% & 58.7\% \\  
\hline
\rowcolor{Gray}
Sparse Combo Net & \checkmark & 130K & \centering --- & \textbf{96.85\%} (4) \newline [\textbf{96.65\%}] & 64.72\% (10) \newline [63.73\%] & 99.04\% & \textbf{96.94\%} & \textbf{65.72\%} \\  
\hline
Lipschitz \newline \citep{erichson2021lipschitz} & \checkmark & 134K & 99.2\% (10) \newline [99.0\%] & 95.9\% (10) \newline [95.6\%] & \centering --- & \textbf{99.4\%} & 96.3\% & 64.2\% \\  
\hline\hline
CKConv \newline \citep{romero2021ckconv} & & 1M & \centering --- & \centering --- & \centering --- & 99.32\% & 98.54\% & 63.74\% \\
\hline
S4 \citep{gu2022s4} & & 7.9M & \centering --- & \centering --- & \centering --- & \textbf{99.63\%} & \textbf{98.7\%} & \textbf{91.13\%} \\
\hline
Trellis \citep{bai2019trellis} & & 8M & \centering --- & \centering --- & \centering --- & 99.2\% & 98.13\% & 73.42\% \\
\hline
\end{tabular}
\caption{Published benchmarks for sequential MNIST, permuted MNIST, and sequential CIFAR10 best test accuracy. Architectures are grouped into three categories: baselines, best performing RNNs with claimed stability guarantee*, and networks achieving overall SOTA. Within each grouping, networks are ordered by number of trainable parameters (for CIFAR10 if it differed across tasks). Our network is highlighted. Where possible, we include information on repeatability.\\
*For more on stability guarantees in machine learning, see Section \ref{section:stability-in-ml}}
\label{table:sota}
\end{table}
\subsubsection{Experiments with Network Size} \label{section: size}
Understanding the effect of size on network performance is important to practical application of these architectures. For both Sparse Combo Net and SVD Combo Net, increasing the number of subnetworks while holding other settings constant (including fixing the size of each subnetwork at 32 units) was able to increase network test accuracy on permuted seqMNIST to a point (Figure \ref{figure:test-sizes}). 

The greatest performance jump happened when increasing from one module (37.1\% Sparse Combo Net, 61.8\% SVD Combo Net) to two modules (89.1\% Sparse Combo Net, 92.9\% SVD Combo Net). After that the performance increased steadily with number of modules until saturating at $\sim 97\%$ for Sparse Combo Net and $\sim 95\%$ for SVD Combo Net. 

As the internal subnetwork weights are not trained in Sparse Combo Net, it is unsurprising that its performance was substantially worse at the smallest sizes. However Sparse Combo Net surpassed SVD Combo Net by the $12 \times 32$ network size, which contains a modest 384 total units. Due to the better performance of the Sparse Combo Net, we focused additional analyses there. Note also that the SVD Combo Net never reached 55\% test accuracy for CIFAR10 in our early experiments. 

We then evaluated task performance as the \emph{modularity} of a Sparse Combo Net (fixed to have 352 total units) was varied. We observed an inverse U shape (Figure \ref{figure:test-sizes-sup}B), with poor performance of a $1 \times 352$ net and an $88 \times 4$ net, and best performance from a $44 \times 8$ net. However, this experiment compared similar sparsity levels, while in practice we can achieve better performance with larger subnetworks by leveraging sparsity in a way not possible for smaller ones.

\begin{figure}[h]
\centering
\includegraphics[width=0.85\textwidth,keepaspectratio]{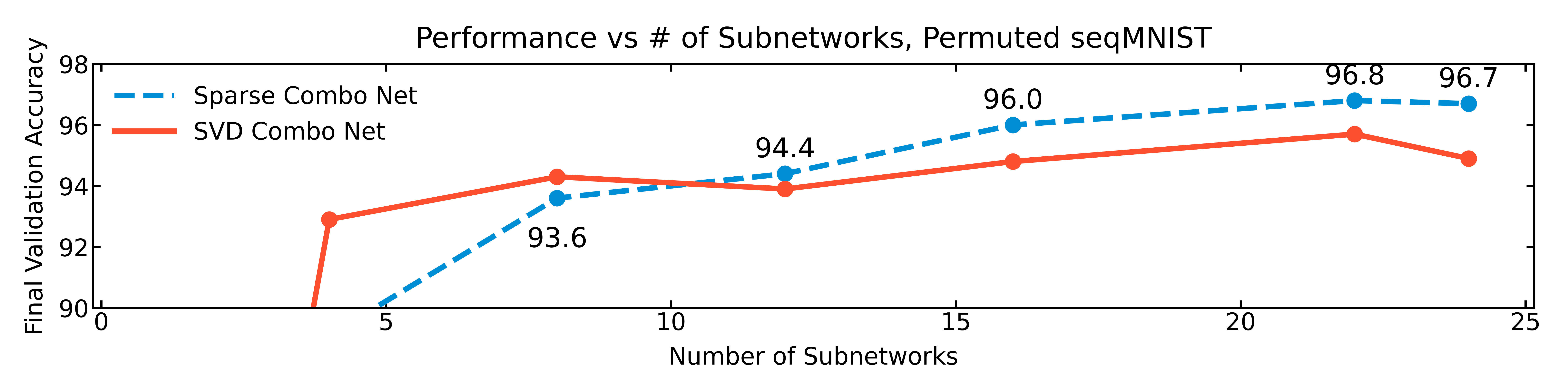}
\caption{Permuted seqMNIST performance plotted against the number of subnetworks. Each subnetwork has 32 neurons. Results are shown for both Sparse Combo Net and SVD Combo Net.}
\label{figure:test-sizes}
\end{figure}

\subsubsection{Experiments with Sparsity}\label{section: sparse}
Because of the link between sparsity and stability as well as the biological relevance of sparsity, we explored in detail how subnetwork sparsity affects the performance of Sparse Combo Net. We ran a number of experiments on the permuted seqMNIST task, varying sparsity level while holding network size and other hyperparameters constant. Here we use "$n\%$ sparsity level" to refer to a network with subnetworks that have just $n\%$ of their weights non-zero. 

We observed a large ($>5$ percentage point) performance boost when switching from a $26.5\%$ sparsity level to a $10\%$ sparsity level in the $11 \times 32$ Sparse Combo Net (Figure \ref{figure:test-sparsity}), and subsequently decided to test significantly sparser subnetworks in a $16 \times 32$ Sparse Combo Net. We trained networks with sparsity levels of $5\%$, $3.3\%$, and $1\%$, as well as $10\%$ for baseline comparison (Figure \ref{figure:test-sparsity-sup}A). A $3.3\%$ sparsity level produced the best results, leading to our SOTA performance for stable networks on both permuted seqMNIST and seqCIFAR10. With a component RNN size of just 32 units, this sparsity level is small, containing only one or two directional connections per neuron on average (Figure \ref{figure:example-best-net}).

As sparsity had such a positive effect on task performance, we did additional analyses to better understand why. We found that decreasing the magnitude of non-zero elements while holding sparsity level constant decreased task performance (Figure \ref{figure:test-sparsity-sup}B), suggesting that the effect is driven in part by the fact that sparsity enables higher magnitude non-zero elements while still maintaining stability. 

\begin{figure}[h]
\centering
\includegraphics[width=0.95\textwidth,keepaspectratio]{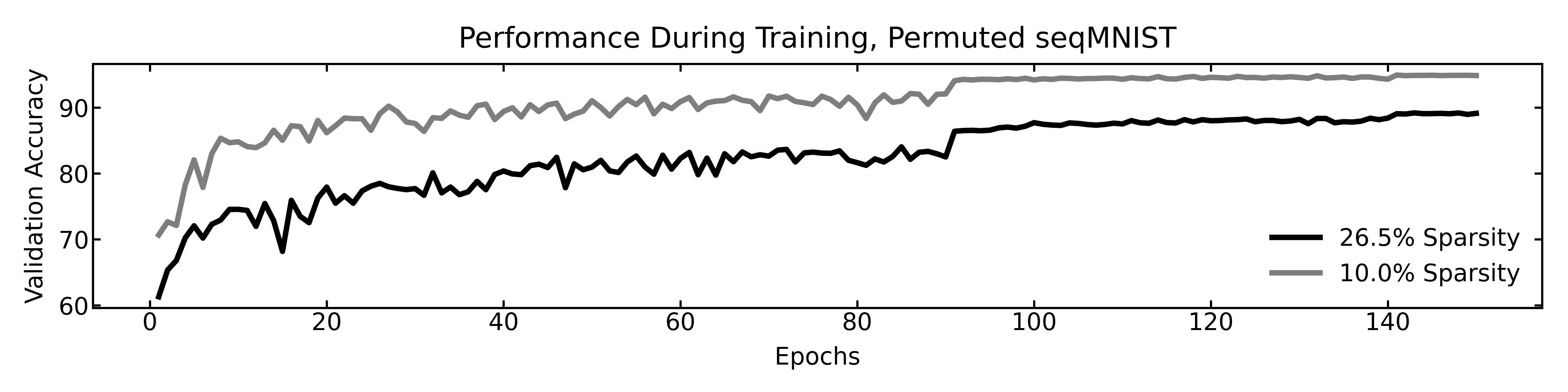}
\caption{Permuted seqMNIST performance over the course of training for two $11 \times 32$ Sparse Combo Nets with different sparsity levels.}
\label{figure:test-sparsity}
\end{figure}

The use of sparsity in subnetworks to improve performance suggests another interesting direction that could enable better scalability of total network size - enforcing sparsity in the linear feedback weight matrix ($\mathbf{L}$). We performed pilot testing of this idea, which showed promise in mitigating the saturation effect seen in Figure \ref{figure:test-sizes}. Those results are detailed in Section \ref{Appendix:scalability-discuss} (Table \ref{table:scalability}). 

\subsubsection{Repeatability and Controls}\label{section:repeats}
Sparse Combo Net does not have the connections within its subnetworks trained, so network performance could be particularly susceptible to random initialization. Thus we ran repeatability studies on permuted sequential MNIST and sequential CIFAR10 using our best network settings ($16 \times 32$ with subnetwork sparsity level of $3.3\%$) and an extended training period. Mean performance over 4 trials of permuted seqMNIST was 96.85\% with 0.019 variance, while mean performance over 10 trials of seqCIFAR10 was 64.72\% with 0.406 variance. Note we also ran a number of additional experiments on size and sparsity settings, described in Section \ref{Appendix:extended-results}.

Across the permuted seqMNIST trials, best test accuracy always fell between $96.65\%$ and $96.94\%$, a range much smaller than the differences seen with changing sparsity settings and network size. Three of the four trials showed best test accuracy $\geq 96.88\%$, despite some variability in early training performance (Figure \ref{figure:test-reproduce}). Similarly, eight of the ten seqCIFAR10 trials had test accuracy $>64.3\%$, with all results falling between 63.73\% and 65.72\% (Figure \ref{figure:cifar-reps}). This robustly establishes a new SOTA for stable RNNs, comfortably beating the previously reported (single run) 64.2\% test accuracy achieved by Lipschitz RNN \citep{erichson2021lipschitz}. 

As a control study, we also tested how sensitive the Sparse Combo Net was to the stabilization condition on the interconnection matrix ($\mathbf{L}$ in \eqref{eq:feedback_combo}). To do so, we initialized the individual RNN modules in a $24 \times 32$ network as before, but set $\mathbf{L}=\mathbf{B}$ and did not constrain $\mathbf{B}$ at all during training, thus no longer ensuring contraction of the overall system. This resulted in 47.0\% test accuracy on the permuted seqMNIST task, a stark decrease from the original 96.7\% test accuracy - thereby demonstrating the utility of the contraction condition.

\section{Discussion}
Biologists have long noted that modularity provides organisms with stability and robustness \citep{Kitano_2004}. The other direction -- that stability and robustness provide modularity -- is well known to engineers \citep{khalil2002nonlinear,slotine1991applied,modular}, but has been less appreciated in biology. We use this principle to build and train provably stable assemblies of recurrent neural networks. Like real brains, the components of our "RNN of RNNs" can communicate with one another through a mix of hierarchical and feedback connections. In particular, we theoretically characterized conditions under which an RNN of RNNs will be stable, given that each individual RNN is stable. We also provided several novel stability conditions for single RNNs that are compatible with these stability-preserving interareal connections. Our results contribute towards understanding how the brain maintains stable and accurate function in the presence of massive interareal feedback, as well as external inputs.

The question of neural stability is one of the oldest questions in computational neuroscience. Indeed, cyberneticists were concerned with this question before the term `computational neuroscience' existed \citep{wiener2019cybernetics,ashby2013design}. Stability is a central component in several influential neuroscience theories \citep{hopfield1982neural,seung1996brain,murphy2009balanced}, perhaps the most well-known being that memories are stored as stable point attractors \citep{hopfield1982neural}. Our work shows that stability continues to be a key concept for computational neuroscience as the field transitions from focusing on single brain areas to many interacting brain areas.

While primarily motivated by neuroscience, our approach is also relevant for machine learning. Deep learning models can be as inscrutable as they are powerful. This opacity limits conceptual progress and may be dangerous in safety-critical applications like autonomous driving or human-centered robotics. Given that stability is a fundamental property of dynamical systems -- and is intimately linked to concepts of control, generalization, efficiency, and robustness -- the ability to guarantee stability of a recurrent model will be important for ensuring deep networks behave as we expect them to \citep{richards2018lyapunov,choromanski2020ode,revay2021recurrent,rodriguez2022lyanet}. 

In the case of RNNs, one difficulty is that providing a certificate of \textit{stability} is often impossible or computationally impractical. However, the stability conditions we derive here allow for recursive construction of complicated RNNs while automatically preserving stability. By parameterizing our conditions for easy optimization using gradient-based techniques, we successfully trained our architecture on challenging sequential processing benchmarks. The high test accuracy our networks achieved with a small number of trainable parameters demonstrates that stability does not necessarily come at the cost of expressivity. Thus, our results likewise contribute towards understanding stability certification of RNNs.   

In future work, we will explore how our contraction-constrained RNNs of RNNs perform on a variety of neuroscience tasks, in particular tasks with multimodal structure \citep{yang2019task}. Our approach is particularly compatible with "global workspace" models, in which different networks communicate via a shared latent space \citep{tabareau2006notes,newell1972human,baars1993cognitive,dehaene1998neuronal,vanrullen2021deep,goyal2021coordination}. One desiderata for these future models is that they learn representations which are formally similar to those observed in the brain \citep{yamins2014performance, schrimpf2020brain, williams2021generalized}, in complement with the structural similarities already shared. Moreover, a "network of networks" approach will be especially relevant to challenging multimodal machine learning problems, such as the simultaneous processing of audio and video. Therefore the advancement of neuroscience theory and machine learning remain hand-in-hand for our next lines of questioning. Indeed, combinations of trained networks have already seen groundbreaking success in DeepMind's AlphaGo \citep{silver2016mastering}. 

As well as the many potential experimental applications, there are numerous theoretical future directions suggested by our work. Networks with more biologically-plausible weight update rules, such as models discussed in \citep{kozachkov2020achieving}, would be a fruitful neuroscience context in which to explore our conditions. One promising avenue of study there is to examine input-dependent stability of the learning process. In the context of machine learning, our stability conditions could be applied to the end-to-end training of multidimensional recurrent neural networks \citep{graves2007multi}, which have clear structural parallels to our RNNs of RNNs but lack known stability guarantees. 

In sum, recursively building network combinations in an effective and stable fashion while also allowing for continual refinement of the individual networks, as nature does for biological networks, will require new analysis tools. Here we have taken a concrete step towards the development of such tools, not only through our theoretical results, but also through their application to create stable combination network architectures that perform well in practice on benchmark tasks.

\begin{ack}
 This work benefited from stimulating discussions with Michael Happ, Quang-Cuong Pham, and members of the Fiete lab at MIT.
\end{ack}


\bibliographystyle{plainnat}
\bibliography{cur_refs}

\newpage


\section*{Checklist}
\begin{enumerate}

\item For all authors...
\begin{enumerate}
  \item Do the main claims made in the abstract and introduction accurately reflect the paper's contributions and scope?
    \answerYes{}
  \item Did you describe the limitations of your work?
    \answerYes{See Supplementary Section \ref{Appendix:Limitations}}
  \item Did you discuss any potential negative societal impacts of your work?
    \answerNo{There are no unique potential negative societal impacts of this work}
  \item Have you read the ethics review guidelines and ensured that your paper conforms to them?
    \answerYes{}
\end{enumerate}

\item If you are including theoretical results...
\begin{enumerate}
  \item Did you state the full set of assumptions of all theoretical results?
    \answerYes{See mathematical details within the paper, Section~\ref{section:model-math} and Section~\ref{section:single-rnn}}
   \item Did you include complete proofs of all theoretical results?
    \answerYes{See Supplementary Section~\ref{Appendix:Proofs}}
\end{enumerate}

\item If you ran experiments...
\begin{enumerate}
  \item Did you include the code, data, and instructions needed to reproduce the main experimental results (either in the supplemental material or as a URL)?
    \answerYes{Our code is linked in Supplementary Section~\ref{Appendix:code-data-details}, and can be run directly on Google Colab. We used only extremely common benchmark image datasets in our experiments, which can be utilized directly via the provided Colab link.}
  \item Did you specify all the training details (e.g., data splits, hyperparameters, how they were chosen)?
    \answerYes{All initialization and training details are reported in Supplementary Section~\ref{Appendix:training-details}. The results for all hyperparameter tuning trials are included in Supplementary Section~\ref{Appendix:all-tables}}
        \item Did you report error bars (e.g., with respect to the random seed after running experiments multiple times)?
    \answerYes{Where applicable in the main text, we reported mean, min, max, and variance of test accuracy across our repeated trials (see Section~\ref{section:repeats} in particular). We also provide extended information on repeatability tests in Supplementary Section~\ref{Appendix:extra-repeats}, and tables that report the results of each individual trial we ran in Supplementary Section~\ref{Appendix:all-tables}}
        \item Did you include the total amount of compute and the type of resources used (e.g., type of GPUs, internal cluster, or cloud provider)?
    \answerYes{See information on the code in Supplementary Section~\ref{Appendix:code-data-details}}
\end{enumerate}

\item If you are using existing assets (e.g., code, data, models) or curating/releasing new assets...
\begin{enumerate}
  \item If your work uses existing assets, did you cite the creators?
    \answerYes{We cite the article that introduced the sequential image classification tasks when introducing them in Section~\ref{section:experiments-intro}}
  \item Did you mention the license of the assets?
    \answerYes{We include the licenses for MNIST and CIFAR10 datasets as part of Supplementary Section~\ref{Appendix:code-data-details}}
  \item Did you include any new assets either in the supplemental material or as a URL?
    \answerYes{Our code is linked in Supplementary Section~\ref{Appendix:code-data-details}}
  \item Did you discuss whether and how consent was obtained from people whose data you're using/curating?
    \answerNA{}
  \item Did you discuss whether the data you are using/curating contains personally identifiable information or offensive content?
    \answerNA{}
\end{enumerate}

\item If you used crowdsourcing or conducted research with human subjects...
\begin{enumerate}
  \item Did you include the full text of instructions given to participants and screenshots, if applicable?
    \answerNA{}
  \item Did you describe any potential participant risks, with links to Institutional Review Board (IRB) approvals, if applicable?
    \answerNA{}
  \item Did you include the estimated hourly wage paid to participants and the total amount spent on participant compensation?
    \answerNA{}
\end{enumerate}

\end{enumerate}


\newpage

\renewcommand\thefigure{S\arabic{figure}}    
\setcounter{figure}{0}  
\renewcommand\thetable{S\arabic{table}}    
\setcounter{table}{0}  
\renewcommand\thesection{A\arabic{section}} 
\setcounter{section}{0}
\addtocontents{toc}{\protect\setcounter{tocdepth}{3}}
\renewcommand\contentsname{Appendix}

\tableofcontents

\newpage

\appendix

\section{Extended Background}

\subsection{Two Different RNNs}
\label{Appendix:Supplementary Math:two_different_rnns}
Note that in neuroscience, the variable $\mathbf{x}$ in equation (\ref{eq:RNN}) is typically thought of as a vector of neural membrane potentials. It was shown in \citep{miller2012mathematical} that the RNN (\ref{eq:RNN}) is equivalent via an affine transformation to another commonly used RNN model,
\begin{equation}\label{eq:yeq}
\tau \dot{\mathbf{y}} = -\mathbf{y} + \phi(\mathbf{W}\mathbf{y} + \mathbf{b}(t))
\end{equation}
where the variable $\mathbf{y}$ is interpreted as a vector of firing rates, rather than membrane potentials. The two models are related by the transformation $\mathbf{x} = \mathbf{W}\mathbf{y} + \mathbf{b}$, which yields
\begin{equation}\nonumber
\tau\dot{\mathbf{x}} = \mathbf{W}(-\mathbf{y} + \phi(\mathbf{W}\mathbf{y} + \mathbf{b})) + \tau \dot{\mathbf{b}}
= -\mathbf{x} + \mathbf{W}\phi(\mathbf{x}) + \mathbf{v}
\end{equation}
where $\mathbf{v}\equiv \mathbf{b} + \tau\dot{\mathbf{b}}$. Thus $\mathbf{b}$ is a low-pass filtered version of $\mathbf{v}$ (or conversely, 
$\mathbf{v}$ may be viewed as a first order prediction of $\mathbf{b}$) and the contraction properties of the system are unaffected by the affine transformation. 
Note that the above equivalence holds even in the case where $\mathbf{W}$ is not invertible. In this case, the two models are proven to be equivalent, provided that $\mathbf{b}(0)$ and $\mathbf{y}(0)$ satisfy certain conditions--which are always possible to satisfy \citep{miller2012mathematical}. Therefore, any contraction condition derived for the $x$ (or $y$) system automatically implies contraction of the other system. We exploit this freedom freely throughout the paper. \\ 

\subsection{Contraction Math}
\label{Appendix:Supplementary Math:supp_math}

It can be shown that the non-autonomous system

\[\dot{\mathbf{x}} = \mathbf{f}(\mathbf{x},t) \]
is contracting if there exists a metric $\mathbf{M}(\mathbf{x},t)  = \mathbf{\Theta}(\mathbf{x},t)^T\mathbf{\Theta}(\mathbf{x},t) \succ 0$ such that uniformly

\[\dot{\mathbf{M}} + \mathbf{M}\mathbf{J} + \mathbf{J}^T\mathbf{M} \preceq -\beta \mathbf{M} \]
where $\mathbf{J} = \frac{\partial \mathbf{f}}{\partial \mathbf{x}}$ and $\beta > 0$. For more details see the main reference \citep{lohmiller1998contraction}. Similarly, a non-autonomous discrete-time system

\[\mathbf{x}_{t+1} = \mathbf{f}(\mathbf{x}_t,t) \]
is contracting if

\[\mathbf{J}^T \mathbf{M}_{t+1} \mathbf{J} - \mathbf{M}_t \preceq -\beta \mathbf{M}_t  \]

\subsubsection{Feedback and Hierarchical Combinations}
\label{Appendix:Supplementary Math:FB}
Consider two systems, independently contracting in constant metrics $\mathbf{M}_1$ and $\mathbf{M}_2$, which are combined in feedback:

\begin{equation*}\tag*{(Feedback Combination)}
\begin{split}
\dot{\mathbf{x}} = \mathbf{f}(\mathbf{x},t) + \mathbf{B}\mathbf{y} \\
\dot{\mathbf{y}} = \mathbf{g}(\mathbf{y},t) + \mathbf{G}\mathbf{x}
\end{split}
\end{equation*}

If the following relationship between $\mathbf{B},\mathbf{G},\mathbf{M}_1$, and $\mathbf{M}_2$ is satisfied:

\begin{equation*}
\mathbf{B} = -\mathbf{M}^{-1}_1\mathbf{G}^T\mathbf{M}_2 
\end{equation*}
then the combined system is contracting as well. This may be seen as a special case of the feedback combination derived in \citep{tabareau2006notes}. The situation is even simpler for hierarchical combinations. Consider again two systems, independently contracting in some metrics, which are combined in hierarchy:

\begin{equation}\tag*{(Hierarchical Combination)}
\begin{split}
\dot{\mathbf{x}} = \mathbf{f}(\mathbf{x},t) \\
\dot{\mathbf{y}} = \mathbf{g}(\mathbf{y},t) + \mathbf{h}(\mathbf{x},t)
\end{split}
\end{equation}

where $\mathbf{h}(\mathbf{x},t)$ is a function with \textit{bounded} Jacobian. Then this combined system is contracting in a diagonal metric, as shown in \citep{lohmiller1998contraction}. By recursion, this extends to hierarchies of arbitrary depth. 

\section{Extended Discussion}
Given the paucity of existing theory on modular networks, our novel stability conditions and proof of concept combination architectures are a significant step in an important new direction. The ``network of networks" approach is evident in the biological brain, and has seen early practical success in applications such as AlphaGo. There is much evidence this line of questioning will be critical in the future, and our work is the first on stable \emph{modular} networks.

Furthermore, we develop an architecture based on such combinations of ``vanilla" RNNs that is both stable and achieves high performance on benchmark sequence classification tasks using few trainable parameters (small particularly for sequential CIFAR10). When considering just the facts about the network, it really has no business performing anywhere near as well as it does. Note also that without the stability condition in place, the network performance indeed drops substantially.

In order to facilitate the extension of this important line of thinking, we provide additional context on the limitations of our current approach as well as promising ideas for future directions in this section.

\subsection{Limitations}\label{Appendix:Limitations}
One drawback of our approach is that we parameterize each weight matrix in a special way to guarantee stability. In all cases, this requires us to parameterize matrices as the product of several other matrices. Thus, we gain stability at the cost of increasing the number of parameters, which can slow down training. 

Another current drawback is that we only consider constant metrics. In theory, contraction metrics can be state-dependent as well as time-varying. Thus it is possible that we have overly restricted the space of models we consider. Similarly, negative feedback is not the only way to preserve contraction when combining two contracting systems. There are known small-gain theorems in the contraction analysis literature which accomplish the same task \citep{modular}. However, parameterizing these conditions is less straightforward than parameterizing the negative-feedback condition. 

A third limitation of the present work is that it does not give a recipe on how to incorporate anatomical knowledge into the building of `RNNs of RNNs'. Our current approach is `bottom up', in the sense that we describe complicated networks which can be built from simpler ones while ensuring stability at every level of construction. However, for building biological models of the brain it is important to incorporate known anatomical detail (i.e V4 projects to PFC, PFC projects back to V4, etc). How to do this in a way that preserves stability is an open and interesting question. 

Lastly, we only tested our networks on sequential image classification benchmarks. Future work will include other benchmarks such as character or word-level language modeling. Additionally, while we conducted preliminary experiments exploring the role of scale (i.e number of subnetwork RNNs), we did not pursue this at the sizes reached by many modern deep learning applications. Thus it is currently unclear if the performance of these stability-constrained models will scale well enough with the number of subnetworks (or the number of neurons per subnetwork). Testing this correctly will require extensive experimentation with the various initialization and training settings. 

\subsection{Future Directions}\label{Appendix:FutureDirs}
There are numerous future directions enabled by this work. For example, Theorem \ref{theorem: Wdiagstabtheorem} suggests that a less restrictive contraction condition on W in terms of the eigenvalues of the symmetric part is possible and desirable. Meanwhile, Theorem \ref{theorem: Wdiagstabcounterexampletheorem} provides important groundwork in finding such a condition, as it shows the need for a time-varying metric. Investigation of input-dependent metrics could be a fruitful next line of research, and would have far-reaching implications in disciplines such as curriculum learning.

Furthermore, the beneficial impact of sparsity on training these stable models suggests a potential avenue for additional experimental work – in particular adding a regularizing sparsity term during training. This could allow Sparse Combo Net to have its internal subnetwork weights trained without losing the stability guarantee, and conversely it could allow SVD Combo Net to reap some of the performance benefits of Sparse Combo Net without giving up the flexibility allowed by training said internal weights. 

As described in the limitations above, a major experimental next step will be to test our architectures at greater scale and on more difficult tasks. Since `network of network’ approaches are becoming increasingly popular, our methodology is relevant to a variety of task types, including reinforcement learning applications. Given the biological inspiration, multi-modal learning tasks may also be of particular relevance.

\section{Proofs for Main Results}
\label{Appendix:Proofs}

\subsection{Proof of Feedback Combination Property}\label{subsection:feedback_combo_proof}
Here we apply existing contraction analysis results to derive equation \eqref{eq:feedback_combo}. Because \eqref{eq:feedback_combo} is a parameterization of a known contraction conditions \citep{modular}, we provide the following statement in the form of a corollary. \\

\begin{corollary}[Network of Networks]
Consider a collection of $p$ subnetwork RNNs governed by \eqref{eq:RNN}. Assume that these RNNs each have hidden-to-hidden weight matrices $\{ \mathbf{W}_1,\dots,\mathbf{W}_p \}$ and are independently contracting in metrics $\{ \mathbf{M}_1,\dots,\mathbf{M}_p \}$. Define the block matrix $\Tilde{\mathbf{W}} \equiv \text{BlockDiag}(\mathbf{W}_1,\dots,\mathbf{W}_p )$ and  $\Tilde{\mathbf{M}} \equiv \text{BlockDiag}(k_1\mathbf{M}_1,\dots,k_p\mathbf{M}_p )$ where $k_i > 0$. Also define the overall state vector $\Tilde{\mathbf{x}}^T \equiv (\mathbf{x}^T_1 \cdots \mathbf{x}^T_p)$, and finally $\Tilde{\mathbf{u}}^T \equiv (\mathbf{u}^T_1 \cdots \mathbf{u}^T_p)$. Finally, the matrix $\Tilde{\mathbf{L}}$ is a block matrix constructed from the matrices $\mathbf{L}_{ij}$ by placing them at the $i,j$ block position of $\Tilde{\mathbf{L}}$. Then if there exists a positive semi-definite matrix $\mathbf{Q}$ such that:
\[\Tilde{\mathbf{M}}\Tilde{\mathbf{L}} + \Tilde{\mathbf{L}}^T\Tilde{\mathbf{M}} = - \mathbf{Q} \]
then the following `network of networks' is globally contracting in metric $\Tilde{\mathbf{M}}$:
\begin{equation}\label{eq:comboRNN}
\begin{split}
\tau \dot{\Tilde{\mathbf{x}}} = -\Tilde{\mathbf{x}} + \Tilde{\mathbf{W}}\phi(\Tilde{\mathbf{x}}) + \Tilde{\mathbf{u}} + \Tilde{\mathbf{L}}\Tilde{\mathbf{x}} 
\end{split}
\end{equation}

\end{corollary}

\begin{proof}

Consider the differential Lyapunov function:

\[V = \frac{1}{2}\delta\mathbf{x}^T\Tilde{\mathbf{M}}\delta\mathbf{x} \]

The time-derivative of this function is:

\[\dot{V} = \delta\mathbf{x}^T\Tilde{\mathbf{M}}\delta\dot{\mathbf{x}} = \delta\mathbf{x}^T (\underbrace{\Tilde{\mathbf{M}}\Tilde{\mathbf{J}} + \Tilde{\mathbf{J}}^T\Tilde{\mathbf{M}}}_{\text{Jacobian of RNNs before interconnection}} + \underbrace{\Tilde{\mathbf{M}}\Tilde{\mathbf{L}} + \Tilde{\mathbf{L}}^T\Tilde{\mathbf{M}}}_{\text{Interconnection Jacobian}} )\delta\mathbf{x}\]

Since we assume that the RNNs are contracting in isolation, the first term in this sum is less that the slowest contracting rate of the individual RNNs, which we call $\lambda > 0$, scaled by the corresponding $k$. Plugging in the definition of $\mathbf{Q}$, we see the second term in the sum is negative semi-definite, by construction. The time-derivative of $V$ is therefore upper-bounded by:
\[\dot{V} \leq -2\lambda V\]
which implies that the network of networks is contracting with rate $\lambda$. 
\end{proof}

\begin{corollary}
If $\Tilde{\mathbf{L}}$ can be written as:
\[ \Tilde{\mathbf{L}} \equiv \mathbf{B} - \Tilde{\mathbf{M}}^{-1}\mathbf{B}^T\Tilde{\mathbf{M}} - \Tilde{\mathbf{M}}^{-1}\mathbf{C} \]
where $\mathbf{B}$ is an arbitrary square matrix, and $\mathbf{C}$ is a matrix whose symmetric part is positive semidefinite, then the above stability criterion is satisfied
\[ \Tilde{\mathbf{M}}\Tilde{\mathbf{L}} + \Tilde{\mathbf{L}}^T\Tilde{\mathbf{M}} = \Tilde{\mathbf{M}}(\mathbf{B} - \Tilde{\mathbf{M}}^{-1}\mathbf{B}^T\Tilde{\mathbf{M}}- \Tilde{\mathbf{M}}^{-1}\mathbf{C}) + (\mathbf{B} - \Tilde{\mathbf{M}}^{-1}\mathbf{B}^T\Tilde{\mathbf{M}}- \Tilde{\mathbf{M}}^{-1}\mathbf{C})^T\Tilde{\mathbf{M}} = - [\mathbf{C}  + \mathbf{C} ^T] = -\mathbf{Q} \]
In this case we may identify $\mathbf{Q} = \mathbf{C} + \mathbf{C}^T$. For the experiments in the main text, we use $\mathbf{C} = \mathbf{0}$, although one could also learn $\mathbf{C}$ via a suitable parametrization.

\end{corollary}

\subsection{Proof of Theorem \ref{theorem: absolutevaluetheorem}}
Our first theorem is motivated by the observation that if the y-system (described in Section \ref{Appendix:Supplementary Math:two_different_rnns}) is to be interpreted as a vector of firing rates, it must stay positive for all time. For a linear, time-invariant system with positive states, diagonal stability is equivalent to stability. Therefore a natural question is if diagonal stability of a linearized y-system implies anything about stability of the nonlinear system. More formally, given an excitatory neural network (i.e $ \forall ij, W_{ij} \geq 0$), if the linear system
\[\dot{\mathbf{x}} = -\mathbf{x} + g\mathbf{W}\mathbf{x} \]
is stable, then there exists a positive diagonal matrix P such that:

\[\mathbf{P}(g\mathbf{W}-\mathbf{I}) +(g\mathbf{W}-\mathbf{I})^T\mathbf{P} \prec 0 \]
The following theorem shows that the nonlinear system (\ref{eq:RNN}) is indeed contracting in metric $\mathbf{P}$, and extends this result to a more general $\mathbf{W}$ by considering only the magnitudes of the weights.

\begin{theorem*}
Let $|\mathbf{W}|$ denote the matrix formed by taking the element-wise absolute value of $\mathbf{W}$.  If there exists a positive, diagonal $\mathbf{P}$ such that:
\[\mathbf{P}(g|\mathbf{W}|-\mathbf{I}) +(g|\mathbf{W}|-\mathbf{I})^T\mathbf{P} \prec 0 \]
then \eqref{eq:RNN} is contracting in metric $\mathbf{P}$. Moreover, if $W_{ii} \leq 0$, then $|W|_{ii}$ may be set to zero to reduce conservatism.
\end{theorem*}

This condition is particularly straightforward in the common special case where the network does not have any self weights, with the leak term driving stability. While it can be applied to a more general $\mathbf{W}$, the condition will of course not be met if the network was relying on highly negative values on the diagonal of $\mathbf{W}$ for linear stability. As demonstrated by counterexample in the proof of Theorem \ref{theorem: absolutevaluetheorem}, it can be impossible to use the same metric $\mathbf{P}$ for the nonlinear RNN in such cases.

Theorem \ref{theorem: absolutevaluetheorem} allows many weight matrices with low magnitudes or a generally sparse structure to be verified as contracting in the nonlinear system \eqref{eq:RNN}, by simply checking a linear stability condition (as linear stability is equivalent to diagonal stability for Metzler matrices too \citep{narendra2010metzler}).

Beyond verifying contraction, Theorem \ref{theorem: absolutevaluetheorem} actually provides a metric, with little need for additional computation. Not only is it of inherent interest that the same metric can be shared across systems in this case, it is also of use in machine learning applications, where stability certificates are becoming increasingly necessary. Critically, it is feasible to enforce the condition during training via L2 regularization on $\mathbf{W}$. More generally, there are a variety of systems of interest that meet this condition but do not meet the well-known maximum singular value condition, including those with a hierarchical structure.

\begin{proof}
Consider the differential, quadratic Lyapunov function:

\[V = \delta\mathbf{x}^T\mathbf{P} \delta\mathbf{x}\]
where $\mathbf{P} \succ 0 $ is diagonal. The time derivative of $V$ is:

\[\dot{V} = 2\delta\mathbf{x}^T\mathbf{P} \dot{\delta\mathbf{x}} = 2\delta\mathbf{x}^T\mathbf{P} \mathbf{J}\delta\mathbf{x} = -2\delta\mathbf{x}^T\mathbf{P} \delta\mathbf{x} + 2\delta\mathbf{x}^T\mathbf{P}\mathbf{W} \mathbf{D}\delta\mathbf{x} \]

where $\mathbf{D}$ is a diagonal matrix such that $\mathbf{D}_{ii} = \frac{d\phi_i}{dx} \geq 0$.  We can upper bound the quadratic form on the right as follows:

\[\begin{split} \delta\mathbf{x}^T\mathbf{P}\mathbf{W} \mathbf{D}\delta\mathbf{x} = \sum_{ij} P_i W_{ij} D_j \delta x_i \delta x_j \leq  \\ \sum_{i}P_i W_{ii} D_i |\delta x_i|^2 + \sum_{ij, i \neq j} P_i |W_{ij}| D_j |\delta x_i| |\delta x_j| \leq g|\delta \mathbf{x}|^T\mathbf{P}|\mathbf{W}||\delta \mathbf{x}| \end{split}\]

If $W_{ii} \leq 0 $, the term $P_i W_{ii} D_i |\delta x_i|^2$ contributes non-positively to the overall sum, and can therefore be set to zero without disrupting the inequality. Now using the fact that $\mathbf{P}$ is positive and diagonal, and therefore $\delta\mathbf{x}^T\mathbf{P} \delta\mathbf{x} = |\delta\mathbf{x}|^T\mathbf{P} |\delta\mathbf{x}|$,  we can upper bound $\dot{V}$ as:

\[\dot{V} \leq |\delta\mathbf{x}|^T (-2\mathbf{P} + \mathbf{P}|\mathbf{W}| + |\mathbf{W}|\mathbf{P})|\delta\mathbf{x}| = |\delta\mathbf{x}|^T[ (\mathbf{P}(|\mathbf{W}|-\mathbf{I}) + (|\mathbf{W}|^T -\mathbf{I})\mathbf{P})]|\delta\mathbf{x}|\]

where $|W|_{ij} = |W_{ij}|$, and $|W|_{ii} = 0$ if $W_{ii} \leq 0$ and $|W|_{ii} = |W_{ii}|$ if $W_{ii} > 0$. This completes the proof.

Note that $\mathbf{W}-\mathbf{I}$ is Metzler, and therefore will be Hurwitz stable if and only if $\mathbf{P}$ exists \citep{narendra2010metzler}. \\

It is also worth noting that highly negative diagonal values in $\mathbf{W}$ will prevent the same metric $\mathbf{P}$ from being used for the nonlinear system. Therefore the method used in this proof cannot feasibly be adapted to further relax the treatment of the diagonal part of $\mathbf{W}$. 

The intuitive reason behind this is that in the symmetric part of the Jacobian, $\frac{\mathbf{P}\mathbf{W}\mathbf{D} + \mathbf{D}\mathbf{W}^{T}\mathbf{P}}{2} - \mathbf{P}$, the diagonal self weights will also be scaled down by small $\mathbf{D}$, while the leak portion $-\mathbf{P}$ remains untouched by $\mathbf{D}$.

Now we actually demonstrate a counterexample, presenting a $2 \times 2$ symmetric Metzler matrix $\mathbf{W}$ that is contracting in the identity in the linear system, but cannot be contracting \textit{in the identity} in the nonlinear system \eqref{eq:RNN}: 

\begin{align*}
\mathbf{W} = \begin{bmatrix} 
    -9 & 2.5 \\
    2.5 & 0 
  \end{bmatrix}
\end{align*}
  
\noindent To see that it is not possible for the more general nonlinear system with these weights to be contracting in the identity, take 
  $\mathbf{D} = \begin{bmatrix} 
    0 & 0 \\
    0 & 1 
  \end{bmatrix}$. Now
  
\begin{align*}
    (\mathbf{W}\mathbf{D})_{sym} - \mathbf{I} = \begin{bmatrix} 
    -1 & 1.25 \\
    1.25 & -1 
  \end{bmatrix}
\end{align*}

\noindent which has a positive eigenvalue of $\frac{1}{4}$.
  
  \end{proof}

\subsection{Proof of Theorem \ref{theorem: symmetricweightstheorem}}
While regularization may push networks towards satisfying Theorem \ref{theorem: absolutevaluetheorem}, strictly enforcing the condition during optimization is not straightforward. This motivates the rest of our theorems, which derive contraction results for specially structured weight matrices. Unlike Theorem \ref{theorem: absolutevaluetheorem}, these results have direct parameterizations which can easily be plugged into modern optimization libraries. \\

\begin{theorem*}
If $\mathbf{W} = \mathbf{W}^T$ and $ \ g\mathbf{W} \prec \mathbf{I}$, and and $ \phi' > 0$ then (\ref{eq:RNN}) is contracting.

\end{theorem*}

When $\mathbf{W}$ is symmetric, (\ref{eq:RNN}) may be seen as a continuous-time Hopfield network. Continuous-time Hopfield networks with symmetric weights were recently shown to be closely related to Transformer architectures \citep{krotov2020large,ramsauer2020hopfield}. Specifically, the dot-product attention rule may be seen as a discretization of the continuous-time Hopfield network with softmax activation function \citep{krotov2020large}.  
Our results here provide a simple sufficient (and nearly necessary, see above remark) condition for global exponential stability of a given \textit{trajectory} for the Hopfield network. In the case where the input into the network is constant, this trajectory is a fixed point. Moreover, each trajectory associated with a unique input is guaranteed to be unique. Finally, we note that our results are flexible with respect to activation functions so long as they satisfy the slope-restriction condition. This flexibility may be useful when, for example, considering recent work showing that standard activation functions may be advantageously replaced by attention mechanisms \citep{dai2020atac}.

\begin{proof}
We begin by writing $\mathbf{W} = \mathbf{R} - \mathbf{P}$ for some unknown $\mathbf{R} = \mathbf{R}^T$ and $\mathbf{P} = \mathbf{P}^T \succ 0$. The approach of this proof is to show by construction that the condition $g\mathbf{W} \prec \mathbf{I}$ implies the existence of an $\mathbf{R}$ and $\mathbf{P}$ such that the system is contracting in metric $\mathbf{P}$. We consider the $y$ version of the RNN, which as discussed above is equivalent to the $x$ version via an affine transformation.

The differential Lyapunov condition associated to the RNN is:

\begin{equation}\label{eq:first_inequality}
\delta \bx ^T [-2\bM + \bM \bD \bW + \bW \bD \bM   + \beta \bM ]\delta \bx \leq 0   
\end{equation}
Where $\bM,\bW \in \mathbb{R}^{n \times n}$. Let us now make the substitution $\bM = \bP$ and $\bW = \bR-\bP$:

\begin{equation}
\delta \bx ^T [-2\bP + \bP\bD(\bR-\bP) + (\bR-\bP)\bD\bP   + \beta \bP ]\delta \bx \leq 0   
\end{equation}
Collecting terms, we get:

\begin{equation}\label{eq:diff_lyap}
\delta \bx ^T [-2\bP + \bP\bD\bR + \bR\bD\bP-2\bP\bD\bP   + \beta \bP ]\delta \bx \leq 0   
\end{equation}
We can rewrite \eqref{eq:diff_lyap} as a quadratic form over a block matrix, as follows:

\begin{equation}\label{eq:diff_lyap_matrix}
\begin{bmatrix} \delta \bx^T & \delta \bx^T \end{bmatrix} \begin{bmatrix} (\beta-2)\bP & \bR\bD\bP  \\ \bP\bD\bR &  -2\bP\bD\bP \end{bmatrix}\begin{bmatrix} \delta \bx \\ \delta \bx \end{bmatrix} \leq 0
\end{equation}
Now the question becomes, when is \eqref{eq:diff_lyap_matrix} satisfied?  One way to ensure that \eqref{eq:diff_lyap_matrix} is satisfied is to ensure that the associated block matrix is always (i.e for all $\bD$) negative semi-definite. In that case the inequality will hold over \textit{all} possible vectors, not just $\begin{bmatrix} \delta \bx & \delta \bx \end{bmatrix}^T$. In other words, the question is now what constraints on the sub-matrices $\bP,\bD$ and $\bR$ ensure that:

\begin{equation}\label{eq:general_block_matrix}
\forall \mathbf{y}  \in \mathbb{R}^{2n}, \hspace{0.5cm} 
\mathbf{y}^T \begin{bmatrix} (2-\beta)\bP & -\bR\bD\bP \\ -\bP\bD\bR &  2\bP\bD\bP \end{bmatrix} \mathbf{y} \geq 0 
\end{equation}
Note that we have multiplied both sides of the inequality by a minus sign. But this is nothing but the definition of a positive semi-definite matrix. Using the Schur complement \citep{gallier2020schur}(Proposition 2.1), we know that the block matrix is positive semi-definite iff $\bP\bD\bP \succ 0 $ and:

\begin{align*}
\begin{split}
(2-\beta)\mathbf{P}  - \mathbf{R}\mathbf{D}\mathbf{P}(2\mathbf{P}\mathbf{D}\mathbf{P})^{-1}\mathbf{P}\mathbf{D}\mathbf{R}  = \\
(2-\beta)\mathbf{P}  - \frac{1}{2}(\mathbf{R}\mathbf{D}\mathbf{R}) \succeq  (2-\beta)\mathbf{P}  - \frac{g}{2}(\mathbf{R}\mathbf{R}) \succeq 0 
\end{split}
\end{align*}
We continue by setting $\mathbf{P} = \gamma^2\mathbf{R}\mathbf{R}$ with $\gamma^2 = \frac{g}{2(2-\beta)}$, so that the above inequality is satisfied. At this point, we have shown that if $\mathbf{W}$ can be written as:

\[\mathbf{W} = \mathbf{R}- \gamma^2\mathbf{R}\mathbf{R}\]
then (\ref{eq:RNN}) is contracting in metric $\mathbf{M} = \gamma^2\mathbf{R}\mathbf{R}$. What remains to be shown is that if the condition:

\[g\mathbf{W} -\mathbf{I} \prec 0 \]
Is satisfied, then this implies the existence of an $\mathbf{R}$ such that the above is true. To show that this is indeed the case, assume that:
\[\frac{1}{4\gamma^2}\mathbf{I} - \mathbf{W} \succeq 0 \]
Substituting in the definition of $\gamma$, this is just the statement that:
\[\frac{2(2-\beta)}{4g}\mathbf{I} - \mathbf{W} \succeq 0 \]
Setting $\beta = 2\lambda > 0 $, this yields:
\[(1-\lambda)\mathbf{I} \succeq g\mathbf{W} \]
Since $\mathbf{W}$ is symmetric by assumption, we have the eigendecomposition:

\[\frac{1}{4\gamma^2}\mathbf{I} - \mathbf{W}= \mathbf{V}(\frac{1}{4\gamma^2}\mathbf{I}-\mathbf{\Lambda})\mathbf{V}^T \]
where $\mathbf{V}^T\mathbf{V} = \mathbf{I}$ and $\mathbf{\Lambda}$ is a diagonal matrix containing the eigenvalues of $\mathbf{W}$. Denote the symmetric square-root of this expression as $\mathbf{S}$:
\[\mathbf{S} = \mathbf{V}\sqrt{(\frac{1}{4\gamma^2}\mathbf{I}-\mathbf{\Lambda})}\mathbf{V}^T =\mathbf{S}^T \]
Which implies that:
\[\frac{1}{4\gamma^2}\mathbf{I} - \mathbf{W} = \mathbf{S}^T\mathbf{S} \]
We now define $\mathbf{R}$ in terms of $\mathbf{S}$ as follows:
\[\mathbf{R} = \frac{1}{\gamma}\mathbf{S} + \frac{1}{2\gamma^2}\mathbf{I}\]
Which means that:
\[\frac{1}{4\gamma^2}\mathbf{I} - \mathbf{W} = (\gamma\mathbf{R} - \frac{1}{2\gamma}\mathbf{I})(\gamma\mathbf{R} - \frac{1}{2\gamma}\mathbf{I}) \]
Expanding out the right side, we get:
\[\frac{1}{4\gamma^2}\mathbf{I} - \mathbf{W} = \gamma^2 \mathbf{R}\mathbf{R} - \mathbf{R} + \frac{1}{4\gamma^2}\mathbf{I} \]
Subtracting $\frac{1}{4\gamma^2}\mathbf{I}$ from both sides yields:
\[\mathbf{W} =  \mathbf{R}-\gamma^2 \mathbf{R}\mathbf{R}\]
As desired.

\end{proof} 

\subsection{Proof of Theorem \ref{theorem: PQPtheorem}}

\begin{theorem*}
If there exists positive diagonal matrices $\mathbf{P}_1$ and $\mathbf{P}_2$, as well as $\mathbf{Q} = \mathbf{Q} ^T \succ 0$ such that
\[ \mathbf{W} = -\mathbf{P}_1 \mathbf{Q} \mathbf{P}_2 \]
then (\ref{eq:RNN}) is contracting in metric $\mathbf{M} = (\mathbf{P}_1 \mathbf{Q} \mathbf{P}_1)^{-1}$.
\end{theorem*}

\begin{proof}
Consider again a differential Lyapunov function:
\[V = \delta \mathbf{x}^T \mathbf{M} \delta \mathbf{x} \]
the time derivative is equal to:

\[\dot{V} = -2V +\delta \mathbf{x}^T\mathbf{M}\mathbf{W}\mathbf{D}\delta\mathbf{x}    \]
Substituting in the definitions of $\mathbf{W}$ and $\mathbf{M}$, we get:
\[\dot{V} = -2V -\delta \mathbf{x}^T\mathbf{P}^{-1}_1 \mathbf{P}_2\mathbf{D}\delta\mathbf{x} \leq -2V     \]
Therefore $V$ converges exponentially to zero.

\end{proof}

\subsection{Proof of Theorem \ref{theorem: triangularweightstheorem}}
\begin{theorem*}
If $g\mathbf{W}-\mathbf{I}$ is triangular and Hurwitz, then (\ref{eq:RNN}) is contracting in a diagonal metric.  

\end{theorem*}

Note that in the case of a triangular weight matrix, the system (\ref{eq:RNN}) may be seen as a feedforward (i.e hierarchical) network. Therefore, this result follows from the combination properties of contracting systems. However, our proof provides a means of explicitly finding a metric for this system.

\begin{proof}
Without loss of generality, assume that $\mathbf{W}$ is lower triangular. This implies that $W_{ij} = 0 $ if $i \leq j$. 
Now consider the generalized Jacobian:

\[\mathbf{F} = -\mathbf{I} + \mathbf{\Gamma} \mathbf{W} \mathbf{D}\mathbf{\Gamma}^{-1} \]
with $\mathbf{\Gamma}$ diagonal and $\Gamma_i = \epsilon^{i}$ where $\epsilon > 0$. Because $\mathbf{\Gamma}$ is diagonal, the generalized Jacobian is equal to:

\[\mathbf{F} = -\mathbf{I} + \mathbf{\Gamma} \mathbf{W}\mathbf{\Gamma}^{-1}\mathbf{D} \]
Now note that:
\[(\mathbf{\Gamma} \mathbf{W}\mathbf{\Gamma}^{-1})_{ij} = \epsilon^i W_{ij} \epsilon^{-j} = W_{ij} \epsilon^{i-j} \] 
Where $i \leq j$, we have $W_{ij} = 0$ by assumption. Therefore, the only nonzero entries are where $i \geq j$. This means that by making $\epsilon$ arbitrarily small, we can make $\mathbf{\Gamma} \mathbf{W}\mathbf{\Gamma}^{-1}$ approach a diagonal matrix with $W_{ii}$ along the diagonal. Therefore, if:
\[\max_i gW_{ii}-1 < 0 \]
the nonlinear system is contracting. Since $\mathbf{W}$ is triangular, $W_{ii}$ are the eigenvalues of $\mathbf{W}$, meaning that this condition is equivalent to $g\mathbf{W}-\mathbf{I}$ being Hurwitz.

\end{proof}

\subsection{Proof of Theorem \ref{theorem: singularvaluetheorem}}
\begin{theorem*}
If there exists a positive diagonal matrix $\mathbf{P}$ such that:

\[g^2\mathbf{W}^T\mathbf{P}\mathbf{W} - \mathbf{P} \prec 0 \]
\noindent then (\ref{eq:RNN}) is contracting in metric $\mathbf{P}$.  
\end{theorem*}

Note that this is equivalent to the discrete-time diagonal stability condition developed in \citep{revay2020contracting}, for a constant metric. Note also that when $\mathbf{M}=\mathbf{I}$, Theorem \ref{theorem: singularvaluetheorem} is identical to checking the maximum singular value of $\mathbf{W}$, a previously established condition for stability of \eqref{eq:RNN}. However a much larger set of weight matrices are found via the condition when $\mathbf{M}=\mathbf{P}$ instead.

\begin{proof}
Consider the generalized Jacobian:

\[\mathbf{F} = \mathbf{P}^{1/2}\mathbf{J} \mathbf{P}^{-1/2} = -\mathbf{I} + \mathbf{P}^{1/2}\mathbf{W}\mathbf{P}^{-1/2}\mathbf{D}  \]

where $\mathbf{D}$ is a diagonal matrix with $\mathbf{D}_{ii} = \frac{d\phi_i}{dx_i} \geq 0$. Using the subadditivity of the matrix measure $\mu_2$ of the generalized Jacobian we get:
\[\mu_2(\mathbf{F}) \leq -1 + \mu_2(\mathbf{P}^{1/2}\mathbf{W}\mathbf{P}^{-1/2}\mathbf{D}) \]

Now using the fact that $\mu_2(\cdot) \leq ||\cdot||_2$ we have:

\[\mu_2(\mathbf{F}) \leq -1 + ||\mathbf{P}^{1/2}\mathbf{W}\mathbf{P}^{-1/2}\mathbf{D})||_2 \leq -1 + g||\mathbf{P}^{1/2}\mathbf{W}\mathbf{P}^{-1/2}||_2 \]

Using the definition of the 2-norm, imposing the condition $\mu_2(\mathbf{F}) \leq 0$ may be written:
\[g^2\mathbf{W}^T\mathbf{P}\mathbf{W} - \mathbf{P} \prec 0 \]
which completes the proof.

\end{proof}

\subsection{Proof of Theorem \ref{theorem: Wdiagstabtheorem}}

\begin{theorem*}
Let $\mathbf{D}$ be a positive, diagonal matrix with $D_{ii} = \frac{d\phi_i}{dx_i}$, and let $\mathbf{P}$ be an arbitrary, positive diagonal matrix. If:

\[ (g\mathbf{W}-\mathbf{I})\mathbf{P} + \mathbf{P}(g\mathbf{W}^T-\mathbf{I}) \preceq -c\mathbf{P} \]
\noindent and 
\[\dot{\mathbf{D}} - cg^{-1}\mathbf{D} \preceq -\beta\mathbf{D} \]

\noindent for $c,\beta > 0$, then (\ref{eq:RNN}) is contracting in metric $\mathbf{D}$ with rate $\beta$.
\end{theorem*}

\begin{proof}
Consider the differential, quadratic Lyapunov function:
\[V = \delta\mathbf{x}^T\mathbf{P}\mathbf{D} \delta\mathbf{x}\]

where $\mathbf{D} \succ 0 $ is as defined above. The time derivative of $V$ is:
\[ \begin{split} \dot{V} = \delta\mathbf{x}^T \mathbf{P}\dot{\mathbf{D}}\delta\mathbf{x} +\delta\mathbf{x}^T (-2\mathbf{P}\mathbf{D} + \mathbf{P}\mathbf{D}\mathbf{W}\mathbf{D} + \mathbf{D}\mathbf{W}^T\mathbf{D}\mathbf{P} )\delta\mathbf{x}  \end{split} \]

The second term on the right can be factored as:

\[\begin{split}\delta\mathbf{x}^T (-2\mathbf{P}\mathbf{D} + \mathbf{P}\mathbf{D}\mathbf{W}\mathbf{D} + \mathbf{D}\mathbf{W}^T\mathbf{D}\mathbf{P} )\delta\mathbf{x} = \\
\delta\mathbf{x}^T\mathbf{D} (-2\mathbf{P}\mathbf{D}^{-1} + \mathbf{P}\mathbf{W} + \mathbf{W}^T\mathbf{P} )\mathbf{D} \delta\mathbf{x} \leq \\
\delta\mathbf{x}^T\mathbf{D} (-2\mathbf{P}g^{-1} + \mathbf{P}\mathbf{W} + \mathbf{W}^T\mathbf{P} )\mathbf{D} \delta\mathbf{x} = \\
\delta\mathbf{x}^T\mathbf{D} [\mathbf{P}(\mathbf{W} - g^{-1}\mathbf{I}) + (\mathbf{W}^T - g^{-1}\mathbf{I})\mathbf{P} ]\mathbf{D} \delta\mathbf{x} \leq \\
-cg^{-1}\delta\mathbf{x}^T\mathbf{P}\mathbf{D}^2 \delta\mathbf{x}
\end{split} \]

where the last inequality was obtained by substituting in the first assumption above. Combining this with the expression for $\dot{V}$, we have:
\[ \begin{split} \dot{V} \leq \delta\mathbf{x}^T \mathbf{P}\dot{\mathbf{D}}\delta\mathbf{x} -cg^{-1}\delta\mathbf{x}^T \mathbf{P}\mathbf{D}^2\delta\mathbf{x}  \end{split} \]

Substituting in the second assumption, we have:
\[ \begin{split} \dot{V} \leq \delta\mathbf{x}^T \mathbf{P}(\dot{\mathbf{D}} -cg^{-1} \mathbf{D}^2)\delta\mathbf{x}  \leq -\beta \delta\mathbf{x}^T\mathbf{P}\mathbf{D}\delta\mathbf{x} = -\beta V \end{split} \]
and thus $V$ converges exponentially to $0$ with rate $\beta$.
\end{proof}

\subsection{Proof of Theorem \ref{theorem: Wdiagstabcounterexampletheorem}}
\begin{theorem*}
Satisfaction of the condition
\begin{align*}
g\mathbf{W}_{sym} - \mathbf{I} \prec 0
\end{align*}
is \textbf{NOT} sufficient to show global contraction of the general nonlinear RNN (\ref{eq:RNN}) in any constant metric. High levels of antisymmetry in $\mathbf{W}$ can make it impossible to find such a metric, which we demonstrate via a $2 \times 2$ counterexample of the form 
\begin{align*}
\mathbf{W} = \begin{bmatrix} 
    0 & -c \\
    c & 0 
  \end{bmatrix}
 \end{align*}
 with $c \geq 2$.
\end{theorem*}

Note that $g\mathbf{W}_{sym} - \mathbf{I} = g\frac{\mathbf{W} + \mathbf{W}^{T}}{2} - \mathbf{I} \prec 0$ is equivalent to the condition for contraction of the system with \textit{linear} activation in the identity metric. 

The main intuition behind this counterexample is that high levels of antisymmetry can prevent a constant metric from being found in the nonlinear system. This is because $\mathbf{D}$ is a diagonal matrix with values between 0 and 1, so the primary functionality it can have in the symmetric part of the Jacobian is to downweight the outputs of certain neurons selectively. In the extreme case of all 0 or 1 values, we can think of this as selecting a subnetwork of the original network, and taking each of the remaining neurons to be single unit systems receiving input from the subnetwork. For a given static configuration of $\mathbf{D}$ (think linear gains), this is a hierarchical system that will be stable if the subnetwork is stable. But as $\mathbf{D}$ can evolve over time when a nonlinearity is introduced, we would need to find a constant metric that can serve completely distinct hierarchical structures simultaneously - which is not always possible.

Put in terms of matrix algebra, D can zero out columns of $\mathbf{W}$, but not their corresponding rows. So for a given weight pair $w_{ij}, w_{ji}$, which has entry in $\mathbf{W}_{sym} = \frac{w_{ij} + w_{ji}}{2}$, if $D_{i} = 0$ and $D_{j} = 1$, the $i,j$ entry in $(\mathbf{W}\mathbf{D})_{sym}$ will be guaranteed to have lower magnitude if the signs of $w_{ij}$ and $w_{ji}$ are the same, but guaranteed to have higher magnitude if the signs are different. Thus if the linear system would be stable based on magnitudes alone $\mathbf{D}$ poses no real threat, but if the linear system requires antisymmetry to be stable, $\mathbf{D}$ can make proving contraction quite complicated (if possible at all). 

\begin{proof}
The nonlinear system is globally contracting in a \textit{constant} metric if there exists a symmetric, positive definite $\mathbf{M}$ such that the symmetric part of the Jacobian for the system, $(\mathbf{M}\mathbf{W}\mathbf{D})_{sym} - \mathbf{M}$ is negative definite uniformly. Therefore $(\mathbf{M}\mathbf{W}\mathbf{D})_{sym} - \mathbf{M} \prec 0$ must hold for all possible $\mathbf{D}$ if $\mathbf{M}$ is a constant metric the system \textit{globally} contracts in with any allowed activation function, as some combination of settings to obtain a particular $\mathbf{D}$ can always be found. 

Thus to prove the main claim, we present here a simple 2-neuron system that is contracting in the identity metric with linear activation function, but can be shown to have no $\mathbf{M}$ that simultaneously satisfies the $(\mathbf{M}\mathbf{W}\mathbf{D})_{sym} - \mathbf{M} \prec 0$ condition for two different possible $\mathbf{D}$ matrices. \\

\noindent To begin, take 
\begin{align*}
    \mathbf{W} = \begin{bmatrix} 
    0 & -2 \\
    2 & 0 
  \end{bmatrix}
\end{align*}

  \noindent Note that any off-diagonal magnitude $\geq 2$ would work, as this is the point at which $\frac{1}{2}$ of one of the weights (found in $\mathbf{W}_{sym}$ when the other is zeroed) will have magnitude too large for $(\mathbf{W}\mathbf{D})_{sym} - \mathbf{I}$ to be stable. \\
  
  \noindent Looking at the linear system, we can see it is contracting in the identity because
  \begin{align*}
     \mathbf{W}_{sym} - \mathbf{I} = \begin{bmatrix} 
    -1 & 0 \\
    0 & -1 
  \end{bmatrix} \prec 0 
  \end{align*}

  \noindent Now consider $(\mathbf{M}\mathbf{W}\mathbf{D})_{sym} - \mathbf{M}$ with $\mathbf{D}$ taking two possible values of 
  \begin{align*}
      \mathbf{D}_{1} = \begin{bmatrix} 
    1 & 0 \\
    0 & 0 
  \end{bmatrix}
  \hspace{5mm} and \hspace{5mm}
  \mathbf{D}_{2} = \begin{bmatrix} 
    0 & 0 \\
    0 & 1 
  \end{bmatrix}
  \end{align*}
   
  \noindent We want to find some symmetric, positive definite 
  $\mathbf{M} = \begin{bmatrix} 
    a & m \\
    m & b 
  \end{bmatrix}$
  such that $(\mathbf{M}\mathbf{W}\mathbf{D}_{1})_{sym} - \mathbf{M}$ and $(\mathbf{M}\mathbf{W}\mathbf{D}_{2})_{sym} - \mathbf{M}$ are both negative definite. \\
  
  \noindent Working out the matrix multiplication, we get 
  \begin{align*}
  (\mathbf{M}\mathbf{W}\mathbf{D}_{1})_{sym} - \mathbf{M} = \begin{bmatrix} 
    2m - a & b - m \\
    b - m & -b 
  \end{bmatrix} 
  \end{align*}
  and 
  \begin{align*}
    (\mathbf{M}\mathbf{W}\mathbf{D}_{2})_{sym} - \mathbf{M} = \begin{bmatrix} 
    -a & -(a + m) \\
    -(a + m) & -2m - b 
  \end{bmatrix}
  \end{align*} \\
  
  \noindent We can now check necessary conditions for negative definiteness on these two matrices, as well as for positive definiteness on $\mathbf{M}$, to try to find an $\mathbf{M}$ that will satisfy all these conditions simultaneously. In this process we will reach a contradiction, showing that no such $\mathbf{M}$ can exist.
  
  A necessary condition for positive definiteness in a real, symmetric $n \times n$ matrix $\mathbf{X}$ is $x_{ii} > 0$, and for negative definiteness $x_{ii} < 0$. Another well known necessary condition for definiteness of a real symmetric matrix is $|x_{ii} + x_{jj}| > |x_{ij} + x_{ji}| = 2|x_{ij}| \hspace{.2cm} \forall i \neq j$. See \citep{wolframPD} for more info on these conditions. \\
  
  \noindent Thus we will require $a$ and $b$ to be positive, and can identify the following conditions as necessary for our 3 matrices to all meet the requisite definiteness conditions:
  \begin{align}
  2m < a \label{D1diag}
  \end{align}
  \begin{align}
  -2m < b \label{D2diag}
  \end{align}
  \begin{align}
  |2m - (a + b)| > 2|b - m| \label{D1offdiag}
  \end{align}
  \begin{align}
  |-2m - (a + b)| > 2|a + m| \label{D2offdiag}
  \end{align}
 
  Note that the necessary condition for $\mathbf{M}$ to be PD, $a + b > 2|m|$, is not listed, as it is automatically satisfied if \eqref{D1diag} and \eqref{D2diag} are. \\
  
  \noindent It is easy to see that if $m = 0$, conditions \eqref{D1offdiag} and \eqref{D2offdiag} will result in the contradictory conditions $a > b$ and $b > a$ respectively, so we will require a metric with off-diagonal elements. To make the absolute values easier to deal with, we will check $m > 0$ and $m < 0$ cases independently.
  
  First we take $m > 0$. By condition \eqref{D1diag} we must have $a > 2m$, so between that and knowing the signs of all unknowns are positive, we can reduce many of the absolute values. Condition \eqref{D1offdiag} becomes $a + b - 2m > |2b - 2m|$, and condition \eqref{D2offdiag} becomes $a + b + 2m > 2a + 2m$, which is equivalent to $b > a$. If $b > a$ we must also have $b > m$, so condition \eqref{D1offdiag} further reduces to $a + b - 2m > 2b - 2m$, which is equivalent to $a > b$. Therefore we have again reached contradictory conditions.
  
  A very similar approach can be applied when $m < 0$. Using condition \eqref{D2diag} and the known signs we reduce condition \eqref{D1offdiag} to $2|m| + a + b > 2b + 2|m|$, i.e. $a > b$. Meanwhile condition \eqref{D2offdiag} works out to $a + b - 2|m| > 2a - 2|m|$, i.e. $b > a$. \\
  
  \noindent Therefore it is impossible for a single constant $\mathbf{M}$ to accommodate both $\mathbf{D}_{1}$ and $\mathbf{D}_{2}$, so that no constant metric can exist for $\mathbf{W}$ to be contracting in when a nonlinearity is introduced that can possibly have derivative reaching both of these configurations. One real world example of such a nonlinearity is ReLU. Given a sufficiently high negative input to one of the units and a sufficiently high positive input to the other, $\mathbf{D}$ can reach one of these configurations. The targeted inputs could then flip at any time to reach the other configuration. \\
  
  \noindent An additional condition we could impose on the activation function is to require it to be a strictly increasing function, so that the activation function derivative can never actually reach 0. We will now show that a very similar counterexample applies in this case, by taking 
    \begin{align*}
      \mathbf{D}_{1*} = \begin{bmatrix} 
    1 & 0 \\
    0 & \epsilon 
  \end{bmatrix}
  \hspace{5mm} and \hspace{5mm}
  \mathbf{D}_{2*} = \begin{bmatrix} 
    \epsilon & 0 \\
    0 & 1 
  \end{bmatrix}
  \end{align*}
  
  \noindent Note here that the $\mathbf{W}$ used above
 produced a $(\mathbf{W}\mathbf{D})_{sym} - \mathbf{I}$ that just barely avoided being negative definite with the original $\mathbf{D}_{1}$ and $\mathbf{D}_{2}$, so we will have to increase the values on the off-diagonals a bit for this next example. In fact anything with magnitude larger than 2 will have some $\epsilon > 0$ that will cause a constant metric to be impossible, but for simplicity we will now take 
 \begin{align*}
 \mathbf{W}_{*} = \begin{bmatrix} 
    0 & -4 \\
    4 & 0 
  \end{bmatrix}
  \end{align*}
  
  \noindent Note that with $\mathbf{W}_{*}$, even just halving one of the off-diagonals while keeping the other intact will produce a $(\mathbf{W}\mathbf{D})_{sym} - \mathbf{I}$ that is not negative definite. Anything less than halving however will keep the identity metric valid. Therefore, we expect that taking $\epsilon$ in $\mathbf{D}_{1*}$ and $\mathbf{D}_{2*}$ to be in the range $0.5 \geq \epsilon > 0$ will also cause issues when trying to obtain a constant metric.
  
  We will now actually show via a similar proof to the above that $\mathbf{M}$ is impossible to find for $\mathbf{W}_{*}$ when $\epsilon \leq 0.5$. This result is compelling because it not only shows that $\epsilon$ does not need to be a particularly small value, but it also drives home the point about antisymmetry - the larger in magnitude the antisymmetric weights are, the larger the $\epsilon$ where we will begin to encounter problems. \\
  
  \noindent Working out the matrix multiplication again, we now get 
  \begin{align*}
  (\mathbf{M}\mathbf{W}_{*}\mathbf{D}_{1*})_{sym} - \mathbf{M} = \begin{bmatrix} 
    4m - a & 2b - m - 2a\epsilon \\
    b - m - 2a\epsilon & -4m\epsilon - b 
  \end{bmatrix}
  \end{align*}
  and 
  \begin{align*}
    (\mathbf{M}\mathbf{W}_{*}\mathbf{D}_{2*})_{sym} - \mathbf{M} = \begin{bmatrix} 
    4m\epsilon - a & -(2a + m - 2b\epsilon) \\
    -(2a + m - 2b\epsilon) & -4m - b 
  \end{bmatrix}
  \end{align*}
  
  \noindent Resulting in two new main necessary conditions:
  \begin{align}
  |4m - a - b - 4m\epsilon| > 2|2b - m - 2a\epsilon| \label{epsD1offdiag}
  \end{align}
  \begin{align}
  |4m\epsilon - a - b - 4m| > 2|2a + m - 2b\epsilon| \label{epsD2offdiag}
  \end{align}
  \noindent As well as new conditions on the diagonal elements:
  \begin{align}
  4m - a < 0 \label{epsD1diag}
  \end{align}
  \begin{align}
  -4m - b < 0 \label{epsD2diag}
 \end{align}
 
 \noindent We will now proceed with trying to find $a,b,m$ that can simultaneously meet all conditions, setting $\epsilon = 0.5$ for simplicity. 
 
 Looking at $m=0$, we can see again that $\mathbf{M}$ will require off-diagonal elements, as condition \eqref{epsD1offdiag} is now equivalent to the condition $a + b > |4b - 2a|$ and condition \eqref{epsD2offdiag} is similarly now equivalent to $a + b > |4a - 2b|$. 
 
 Evaluating these conditions in more detail, if we assume $4b > 2a$ and $4a > 2b$, we can remove the absolute value and the conditions work out to the contradicting $3a > 3b$ and $3b > 3a$ respectively. As an aside, if $\epsilon > 0.5$, this would no longer be the case, whereas with $\epsilon < 0.5$, the conditions would be pushed even further in opposite directions. 
 
 If we instead assume $2a > 4b$, this means $4a > 2b$, so the latter condition would still lead to $b > a$, contradicting the original assumption of $2a > 4b$. $2b > 4a$ causes a contradiction analogously. Trying $4b = 2a$ will lead to the other condition becoming $b > 2a$, once again a contradiction. Thus a diagonal $\mathbf{M}$ is impossible \\
 
 \noindent So now we again break down the conditions into $m > 0$ and $m < 0$ cases, first looking at $m > 0$. Using condition \eqref{epsD1diag} and knowing all unknowns have positive sign, condition \eqref{epsD1offdiag} reduces to $a + b - 2m > |4b - 2(a + m)|$ and condition \eqref{epsD2offdiag} reduces to $a + b + 2m > |4a - 2(b - m)|$. This looks remarkably similar to the $m = 0$ case, except now condition \eqref{epsD1offdiag} has $-2m$ added to both sides (inside the absolute value), and condition \eqref{epsD2offdiag} has $2m$ added to both sides in the same manner. If $4b > 2(a + m)$ the $-2m$ term on each side will simply cancel, and similarly if $4a > 2(b - m)$ the $+2m$ terms will cancel, leaving us with the same contradictory conditions as before. 
 
 Therefore we check $2(a + m) > 4b$. This rearranges to $2a > 2(2b - m) > 2(b - m)$, so that from condition \eqref{epsD2offdiag} we get $b > a$. Subbing condition \eqref{epsD1diag} in to $2(a + m) > 4b$ gives $8b < 4a + 4m < 5a$ i.e. $b < \frac{5}{8}a$, a contradiction. The analogous issue arises if trying $2(b - m) > 4a$. Trying $2(a + m) = 4b$ gives $m = 2b - a$, which in condition \eqref{epsD2offdiag} results in $5b - a > |6a - 6b|$, while in condition \eqref{epsD1diag} leads to $5a > 8b$, so \eqref{epsD2offdiag} can further reduce to $5b - a > 6a - 6b$ i.e. $11b > 7a$. But $b > \frac{7}{11}a$ and $b < \frac{5}{8}a$ is a contradiction. Thus there is no way for $m > 0$ to work. \\
 
 \noindent Finally, trying $m < 0$, we now use condition \eqref{epsD2diag} and the signs of the unknowns to reduce condition \eqref{epsD1offdiag} to $a + b + 2|m| > |4b - 2(a - |m|)|$ and condition \eqref{epsD2offdiag} to $a + b - 2|m| > |4a - 2(b + |m|)|$. These two conditions are clearly directly analogous to in the $m > 0$ case, where $b$ now acts as $a$ with condition \eqref{epsD2diag} being $b > 4|m|$. Therefore the proof is complete.
\end{proof}

\section{Sparse Combo Net Details}\label{Appendix:Experiments}
Here we provide comprehensive information on the methodology and results for Sparse Combo Net, including some supplementary experimental results. See the Appendix table of contents for a guide to this section. 

\subsection{Extended Methods}
\subsubsection{Initialization and Training}\label{Appendix:training-details}
As described in the main text, the nonlinear RNN weights for Sparse Combo Net were randomly generated based on given sparsity and entry magnitude settings, and then confirmed to meet the Theorem \ref{theorem: absolutevaluetheorem} condition (or discarded if not). For a sparsity level of $x$ and a magnitude limit of $y$, each subnetwork $\mathbf{W}$ was generated by drawing uniformly from between $-y$ and $y$ with $x\%$ density using scipy.sparse.random, and then zeroing out the diagonal entries. For various potential $x$ and $y$ settings, we quantified both the likelihood that a generated $\mathbf{W}$ would satisfy Theorem \ref{theorem: absolutevaluetheorem}, and the resulting network performance. Of course this is also dependent on subnetwork size, as larger subnetworks enable greater sparsity. The information we have obtained so far is documented in Section \ref{Appendix:extended-results}, in particular \ref{Appendix:extra-sparsity-results}.

In training the linear connections between the described nonlinear RNN subnetworks, we constrained the matrix $\mathbf{B}$ in \eqref{eq:feedback_combo} to reflect underlying modularity assumptions. In particular, we only train the off-diagonal blocks of $\mathbf{B}$ and mask the diagonal blocks. We do this to maintain the interpretation of $\mathbf{L}$ as the matrix containing the connection weights \textit{between} different modules, as diagonal blocks would correspond to self-connections. Furthermore, we only train the lower-triangular blocks of $\mathbf{B}$ while masking the others, to increase training speed. 

To obtain the subnetwork RNN metrics necessary for training these linear connections, scipy.integrate.quad was used with default settings to solve for $\mathbf{M}$ in the equation $-\mathbf{I} = \mathbf{M}\mathbf{W} + \mathbf{W}^{T}\mathbf{M}$, as described in the main text. This was done by integrating $e^{\mathbf{W}^{T} t} \mathbf{Q} e^{\mathbf{W} t} dt$ from 0 to $\infty$. For efficiency reasons, and due to the guaranteed existence of a diagonal metric in the case of Theorem \ref{theorem: absolutevaluetheorem}, integration was only performed to solve for the diagonal elements of $\mathbf{M}$. Therefore a check was added prior to training to confirm that the initialized network indeed satisfied Theorem \ref{theorem: absolutevaluetheorem} with metric $\mathbf{M}$. However, it was never triggered by our initialization method. 

Initial training hyperparameter tuning was done primarily with $10 \times 16$ combination networks on the permuted seqMNIST task, starting with settings based on existing literature on this task, and verifying promising settings using a $15 \times 16$ network (Table \ref{table:hyperparams}). Initialization settings were held the same throughout, as was later done for the size comparison trials (described in Section \ref{Appendix:extra-size-results}).

Once hyperparameters were decided upon, the trials reported on in the main text began. Most of these experiments were also done on permuted seqMNIST, where we characterized performance of networks with different sizes and sparsity levels/entry magnitudes. When we moved to the sequential CIFAR10 task, we began by simply training with the same best settings that were found from these experiments. The results of all attempted trials are reported in Section \ref{Appendix:all-tables}.  

Unless specified otherwise, all networks reported on in the main text were trained for 150 epochs, using an Adam optimizer with initial learning rate 1e-3 and weight decay 1e-5. The learning rate was cut to 1e-4 after 90 epochs and to 1e-5 after 140. After identifying the most promising settings, we ran repetitions trials on the best networks for 200 epochs with learning rate cuts after epochs 140 and 190. 

\subsubsection{Code and Datasets}\label{Appendix:code-data-details}
All Sparse Combo Nets described in the main text were trained using a single GPU on Google Colab. Code to replicate all experiments can be found here: \url{https://colab.research.google.com/drive/1JCT5OMgaMVK_Xh8BDFNRrEsyF0Ojvg10?usp=sharing}

Runtime for the best performing architecture settings on the sequential CIFAR10 task was $\sim 24$ hours. A Colab Pro+ account was used to limit timeouts and prioritize GPU access.

The datasets we used for our tasks were MNIST and CIFAR10, downloaded via PyTorch. MNIST is a handwritten digits classification task, consisting of 60,000 training images and 10,000 testing images (each 28x28 and grayscale). It is made available under the terms of the Creative Commons Attribution-Share Alike 3.0 license. CIFAR10 is a dataset of 32x32 color images, split evenly among the following 10 classes: airplane, automobile, bird, cat, deer, dog, frog, horse, ship, and truck. It also contains 60,000/10,000 training/test images, and is distributed under the MIT License.

As mentioned in the main text, we presented these images to our networks pixel by pixel. In the case of MNIST, we used an additional modification to the dataset by permuting the pixels in a randomly determined (but fixed across the dataset) way. The use of these datasets is included in the above link to our code. \\

Extended results information begins on the next page.

\newpage

\subsection{Extended Results}\label{Appendix:extended-results}
In this subsection, we describe additional results we did not get to in the main text, related to scalability, modularity, sparsity, and repeatability. We also add some further discussion of these results, along with more detailed information on the respective experimental set ups. 

\subsubsection{Network Size and Modularity Comparison}\label{Appendix:extra-size-results}
For the Sparse Combo Net specifically we had additional experiments on architecture size and unit distribution besides what was depicted in the main text. The results of these supplemental experiments both replicated the observed effect in Figure \ref{figure:test-sizes} of network size instead using subnetworks with 16 units each as well as higher density (Figure \ref{figure:test-sizes-sup}A), and evaluated how task performance varies with modularity of a network fixed to have 352 total units (Figure \ref{figure:test-sizes-sup}B). In the modularity experiment we observed an inverse U shape, with poor performance of a $1 \times 352$ net and an $88 \times 4$ net, and best performance from a $44 \times 8$ net. Note that this experiment compared similar sparsity levels across the different subnetwork sizes. In practice we can achieve better performance with larger subnetworks by leveraging sparsity in a way not possible in smaller subnetworks. These additional experiments are now described in more detail below.

For the initial round of size comparison trials using subnetworks of 16 units each (Figure \ref{figure:test-sizes-sup}A), the nonlinear RNN weights were set by drawing uniformly from between $-0.4$ and $0.4$ with $40\%$ density using scipy.sparse.random, and then zeroing out the diagonal entries. These settings were chosen because they resulted in $\sim1\%$ of $16$ by $16$ weight matrices meeting the Theorem \ref{theorem: absolutevaluetheorem} condition. During initialization only the matrices meeting this condition were kept, finishing when the desired number of component RNNs had been set - producing a block diagonal $\mathbf{W}$ like pictured in Figure \ref{figure:example-training}A. This same initialization process was used throughout our experiments. In later experiments we vary the density and magnitude settings.

For the size experiments, we held static the number of units and initialization settings for each component RNN, and tested the effect of changing the number of components in the combination network. 1, 3, 5, 10, 15, 20, 22, 25, and 30 components were tested in this experiment (Figure \ref{figure:test-sizes-sup}A). Increasing the number of components initially lead to great improvements in test accuracy, but had diminishing returns - test accuracy consistently hit $\sim93\%$ with a large enough number of subnetworks, but neither loss nor accuracy showed meaningful improvement past the $22 \times 16$ network. Interestingly, early training loss and accuracy became substantially worse once the number of components increased past a certain point, falling from $70\%$ to $43\%$ epoch 1 test accuracy between the $22 \times 16$ and $30 \times 16$ networks. The complete set of results can be found in Table \ref{table:sizes-compare}. 

Note that the size experiment described in the main text (Figure \ref{figure:test-sizes}A) was a repetition of this original experiment, but now using 32 unit subnetworks with the best performing sparsity settings. The results for the repetition can be found in Table \ref{table:compare-sizes-again}.

To better understand how the modularity of the combination networks affects performance, the next experiment held the number of total units constant at 352, selected due to the prior success of the $22 \times 16$ network, and tested different allocations of these units amongst component RNNs. Thus $1 \times 352$, $11 \times 32$, $44 \times 8$, and $88 \times 4$ networks were trained to compare against the $22 \times 16$ (Figure \ref{figure:test-sizes-sup}B). Increasing the modularity improved performance to a point, with the $44 \times 8$ network resulting in final test accuracy of $94.44\%$, while conversely the $11 \times 32$ resulted in decreased test accuracy. However, the $88 \times 4$ network was unable to learn, and a $352 \times 1$ network would theoretically just be a scaled linear anti-symmetric network. 

Because larger networks require different sparsity settings to meet the Theorem \ref{theorem: absolutevaluetheorem} condition, these were not held constant between trials in the modularity comparison experiment (Figure \ref{figure:test-sizes-sup}B), but rather selected in the same way between trials - looking for settings that keep density and scalar balanced and result in $\sim1\%$ of the matrices meeting the condition. The scalar was applied after sampling non-zero entries from a uniform distribution between -1 and 1. The resulting settings were 7.5\% density and 0.077 scalar for 352 unit component RNN, 26.5\% density and 0.27 scalar for 32 unit component RNN, 60\% density and 0.7 scalar for 8 unit component RNN, and 100\% density and 1.0 scalar for 4 unit component RNN. The complete set of results for the modularity experiment can be found in Table \ref{table:modules-compare}.

\begin{figure}[h]
\centering
\includegraphics[width=0.7\textwidth,keepaspectratio]{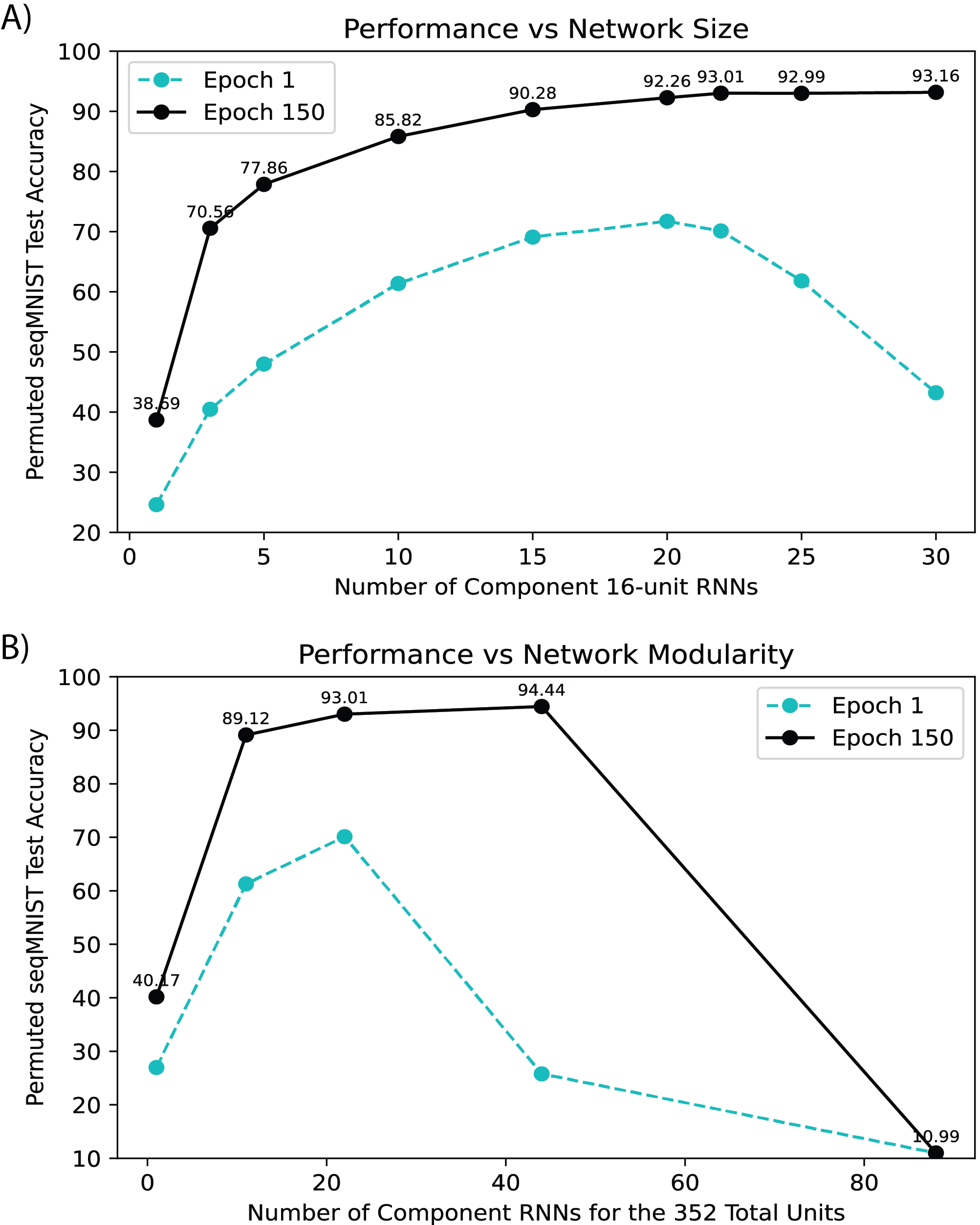}
\caption{Performance of Sparse Combo Nets on the Permuted seqMNIST task by combination network size. We test the effects on final and first epoch test accuracy of both total network size and network modularity. The former is assessed by varying the number of subnetworks while each subnetwork is fixed at 16 units (A), and the latter by varying the distribution of units across different numbers of subnetworks with the total sum of units in the network fixed at 352 (B). Note that these experiments were run prior to optimizing the sparsity initialization settings. Experiments on total network size were later repeated with the final sparsity settings (Figure \ref{figure:test-sizes}A). The results of both the size experiments are consistent.}
\label{figure:test-sizes-sup}
\end{figure}

\subsubsection{Sparsity Settings Comparison}\label{Appendix:extra-sparsity-results}
Density and scalar settings for the component nonlinear RNNs were initially chosen for each network size using the percentage of random networks that met the Theorem \ref{theorem: absolutevaluetheorem} condition. For scalar $s$, a component network would have non-zero entries sampled uniformly between $-s$ and $s$. 

When we began experimenting with sparsity in the initialization (after seeing the large performance difference in Figure \ref{figure:test-sparsity}), we split the previously described scalar setting into two different scalars - one applied before a random matrix was checked against the Theorem \ref{theorem: absolutevaluetheorem} condition, and one applied after a matrix was selected. Of course the latter must be $\leq 1$ to guarantee stability is preserved. The scalar was separated out after we noticed that at $5\%$ density, random $32$ by $32$ weight matrices met the condition roughly $1\%$ of the time whether the scalar was $10$ or $100000$ - $\sim85\%$ of sampled matrices using scalar 10 would continue to meet the condition even if multiplied by a factor of $10000$. Therefore we wanted a mechanism that could bias selection towards matrices that are stable due to their sparsity and not due to magnitude constraints, while still keeping the elements to a reasonable size for training purposes.

Ultimately, both sparsity and magnitude had a clear effect on performance (Figure \ref{figure:test-sparsity-sup}). Increases in both had a positive correlation with accuracy and loss through most parameters tested. Best test accuracy overall was $96.79\%$, which was obtained by both a $16 \times 32$ network with $5\%$ density and entries between -5 and 5, and a $16 \times 32$ network with $3.3\%$ density and entries between -6 and 6. The latter also achieved the best epoch 1 test accuracy observed of $86.79\%$. Thus we chose to go with these settings for our extended training repetitions on permuted seqMNIST. 

It is also worth noting that upon investigation of the subnetwork weight matrices across these trials, the sparser networks had substantially lower maximum eigenvalue of $\mathbf{|W|}$, suggesting that stronger stability can actually correlate with \textit{improved} performance on sequential tasks. This could be due to a mechanism such as that described in \citep{uhler2020assoc}.

Results from the initial trials of different sparsity levels across different network sizes can be found in Table \ref{table:sparsity-compare}. Results from the more thorough testing of different sparsity levels and magnitudes in a $16 \times 32$ network can be found in Table \ref{table:best-results}.

\begin{figure}[h]
\centering
\includegraphics[width=\textwidth,keepaspectratio]{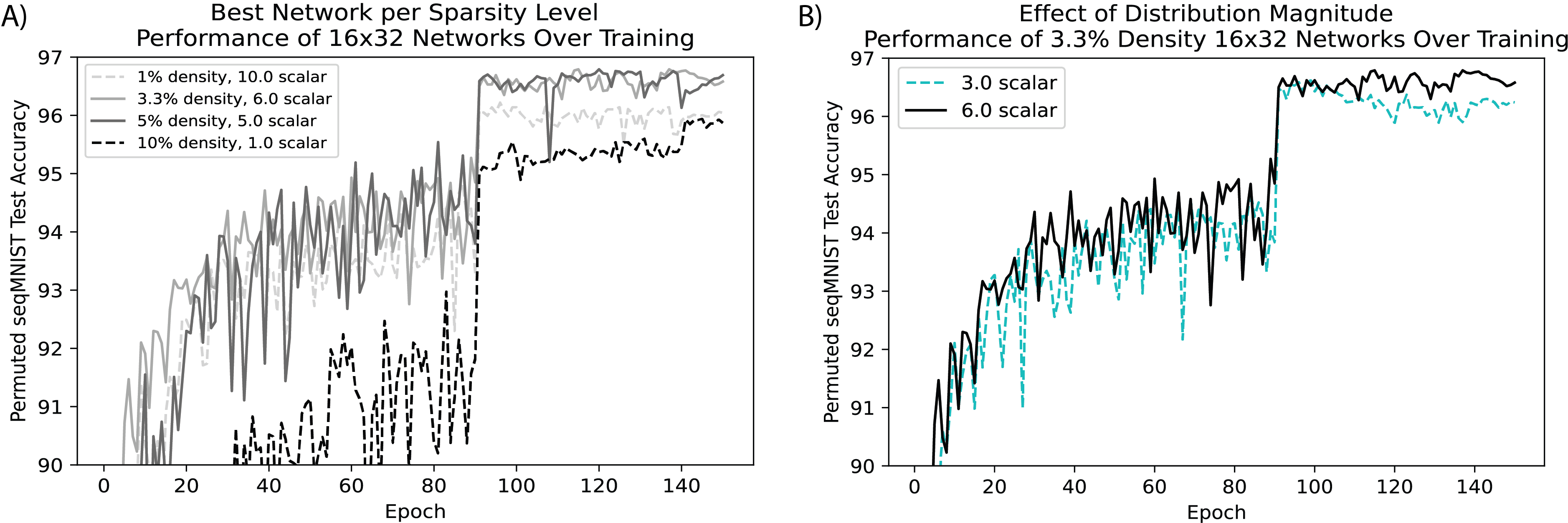}
\caption{Permuted seqMNIST performance by component RNN initialization settings. Test accuracy is plotted over the course of training for four $16 \times 32$ networks with different density levels and entry magnitudes (A), highlighting the role of sparsity in network performance. Test accuracy is then plotted over the course of training for two 3.3\% density $16 \times 32$ networks with different entry magnitudes (B), to demonstrate the role of the scalar. When the magnitude becomes too high however, performance is out of view of the current axis limits.}
\label{figure:test-sparsity-sup}
\end{figure}

\subsubsection{Repeatability Tests}\label{Appendix:extra-repeats}
Here we give additional details on the repeatability tests described in the main text, as well as introducing some additional results using other hyperparameters (done after the initial SOTA-setting experiments) on sequential CIFAR10.

To further improve performance once network settings were explored on permuted seqMNIST, an extended training run was tested on the best performing option. Settings were kept the same as above using a $3.3\%$ density $16 \times 32$ network, except training now ran for over 200 epochs, with just a single learning rate cut occurring after epoch 200 (exact number of epochs varied based on runtime limit). This experiment was repeated four times and resulted in $96.94\%$ best test accuracy, as described in the main text (Figure \ref{figure:test-reproduce}). Table \ref{table:psmnist-repeats} reports additional details on these trials.

\begin{SCfigure}[0.5][h]
\centering
\caption{Permuted seqMNIST performance on repeated trials. Four different $16 \times 32$ networks with 3.3\% density and entries between -6 and 6 were trained for 24 hours, with a single learning rate cut after epoch 200. (A) depicts test accuracy for each of the networks over the course of training. (B) depicts the training loss for the same networks. Exact numbers are reported in Table \ref{table:psmnist-repeats}.}
\includegraphics[width=0.6\textwidth,keepaspectratio]{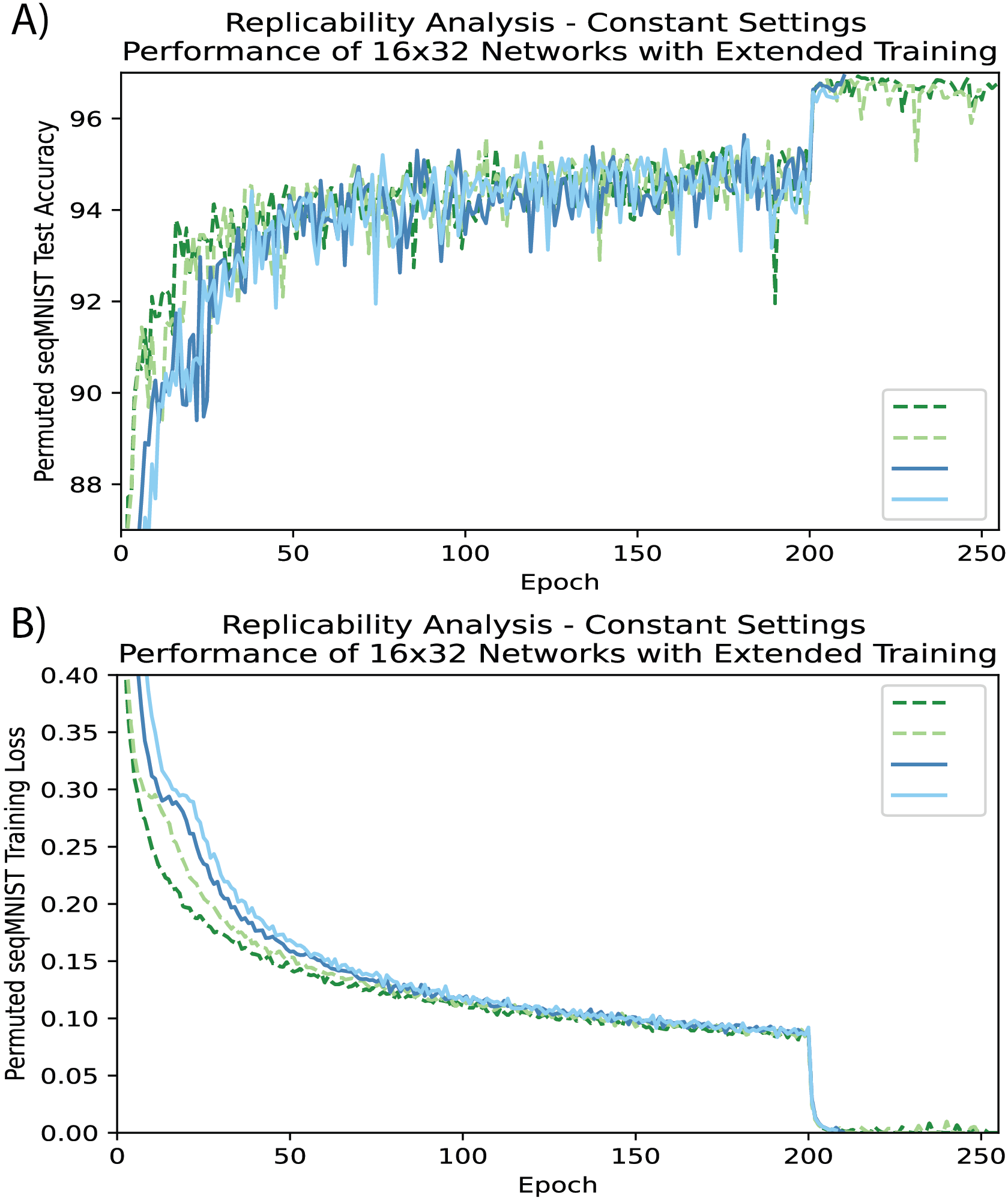}
\label{figure:test-reproduce}
\end{SCfigure}

As mentioned, we also characterized Sparse Combo Net performance on the more challenging CIFAR10 sequential image classification task, across a greater number of trials. Ten trials were run for 200 epochs with learning rate scaled after epochs 140 and 190. All other settings were the same as for the permuted seqMNIST repeatability trials. As reported in the main text, the mean test accuracy observed was 64.72\%, with variance 0.406, and range 63.73\%-65.72\%. Training loss and test accuracy over the course of learning are plotted for all ten trials in Figure \ref{figure:cifar-reps}, and exact numbers for each trial are provided in Table \ref{table:cifar-repeats}. 

As we encountered difficulties with Colab GPU assignment when we started working on these repetitions, we also trained nine networks over a smaller number of epochs (Figure \ref{figure:cifar-reps-shorter}). These networks were trained using the 150 epoch paradigm previously described, although only four of the nine completed training within the 24 hour runtime limit. Complete results for these trials can be found in Table \ref{table:cifar-repeats-short}. Mean performance among the shorter training trials was 62.82\% test accuracy with variance 0.95.

Prior to beginning the repeatability experiments on seqCIFAR10, we explored alternative hyperparameters on this task (Table \ref{table:cifar}). While we ultimately ended up using the same hyperparameters as for permuted seqMNIST, these results further support the robustness of our architecture.

\begin{figure}[h]
\centering
\includegraphics[width=\textwidth,keepaspectratio]{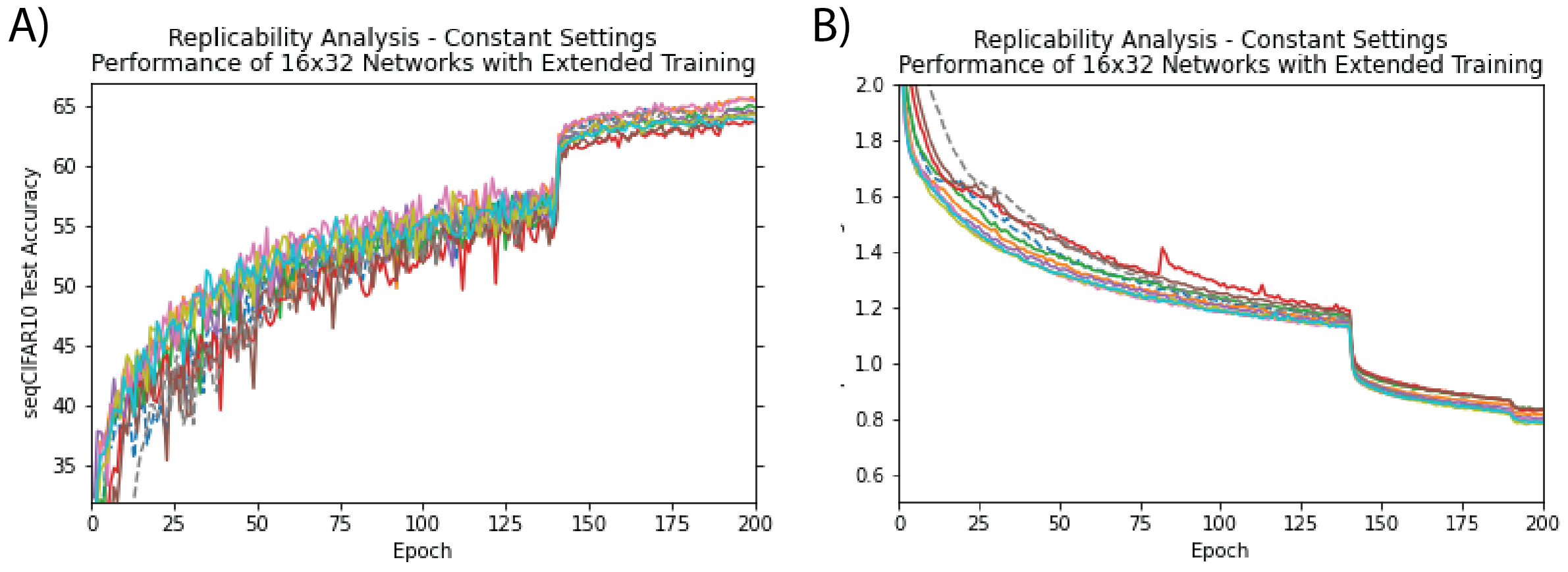}
\caption{seqCIFAR10 performance on repeated trials. Ten different $16 \times 32$ networks with 3.3\% density and entries between -6 and 6 were trained for 200 epochs, with learning rate divided by 10 after epochs 140 and 190. (A) depicts test accuracy for each of the networks over the course of training. (B) depicts the training loss for the same networks. Exact numbers are reported in Table \ref{table:cifar-repeats}.}
\label{figure:cifar-reps}
\end{figure}

\begin{SCfigure}[0.7][h]
\centering
\includegraphics[width=0.6\textwidth,keepaspectratio]{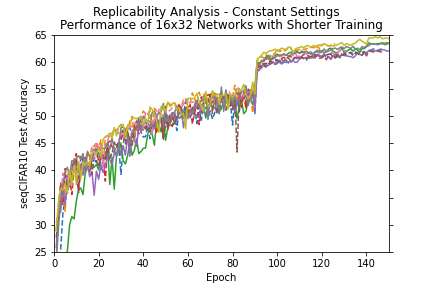}
\caption{seqCIFAR10 performance on repeated trials with shorter training (done to complete more trials). Nine different $16 \times 32$ networks with 3.3\% density and entries between -6 and 6 were set up to train for 150 epochs, with learning rate divided by 10 after epochs 90 and 140. Most of these networks hit runtime limit before completing, however they all got through at least 100 epochs and all had test accuracy exceed 61\%. This figure depicts test accuracy for each of the networks over the course of training. Networks that completed training are plotted as solid lines, while those that were cut short are dashed.}
\label{figure:cifar-reps-shorter}
\end{SCfigure}

To complete our benchmarking table, we also ran a single 150 epoch trial of our best network settings on the sequential MNIST task. Test accuracy over the course of training for this trial is depicted in Figure \ref{figure:seqmnist}.

\begin{SCfigure}[0.3][h]
\centering
\includegraphics[width=0.6\textwidth,keepaspectratio]{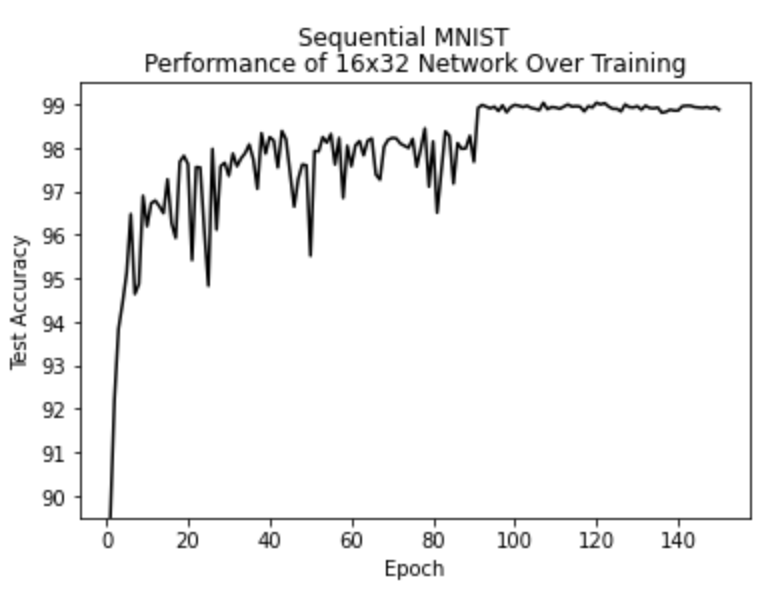}
\caption{Performance over training on the seqMNIST task for a $16 \times 32$ network with best settings (using 150 epoch training protocol). Final test accuracy exceeded 99\%.}
\label{figure:seqmnist}
\end{SCfigure}

\subsubsection{Scalability Pilot: Feedback Sparsity Tests}\label{Appendix:scalability-discuss}
Not only did the $16 \times 32$ Sparse Combo Net with $3.3\%$ density achieve test accuracy sequential CIFAR10 that is the highest to date for a provably stable RNN, it was higher than the 1 million parameter CKConv network, which set a recent SOTA for permuted seqMNIST accuracy (Table \ref{table:sota}). Our network has $130,000$ trainable parameters by comparison. Thus we achieve very impressive results given the characteristics of our architecture.

However, the results of the network size experiments do raise some concern about the scalability of the approach. Here we provide pilot results suggesting that this issue can likely be addressed via the introduction of sparsity in the linear inter-subnetwork connectivity matrix $\mathbf{L}$. While all results in the main text use all-to-all negative feedback, we have early results suggesting that in larger networks, fewer negative feedback connections may perform better, thus preventing performance saturation with scaling. 

Below are the results of testing of this idea in a $24 \times 32$ Sparse Combo Net on the sequential CIFAR10 task. We varied the number of feedback connections that were fixed at 0 while all other settings remained static (Table \ref{table:scalability}). The resulting performance took an inverse U shape, where the network with only 50\% of possible feedback connections non-zero had the best test accuracy of 65.14\%, achieved in just 124 epochs of training.

\begin{table}[!htbp]
\centering
\begin{tabular}{ | m{1cm} | m{1.5cm} || m{1cm} | m{2cm} | m{2cm} | }
\hline
Size & Feedback Density & Epochs & Best Overall Test Acc. & Best Test Acc. Through 85 Epochs \\
\hline\hline
$24 \times 32$ & 100\% & 86 & 52.7\% & 52.7\%  \\  
\hline
$24 \times 32$ & 75\% & 88 & 56.49\% & 56.48\%  \\  
\hline
$24 \times 32$ & 66.6\% & 89 & 58.84\% & 58.84\%  \\  
\hline
\rowcolor{LightCyan}
$24 \times 32$ & 50\% & 124 & 65.14\% & 58.01\%  \\   
\hline
$24 \times 32$ & 33.3\% & 129 & 61.86\% & 56.05\%  \\ 
\hline
$24 \times 32$ & 25\% & 92 & 54.26\% & 50.54\%  \\ 
\hline
$24 \times 32$ & 0\% & 130 & 39.8\% & 38.38\%  \\ 
\hline\hline
\rowcolor{Gray}
$16 \times 32$ & 100\% & 150 & 64.63\% & 55.82\%  \\ 
\hline
$16 \times 32$ & 75\% & 150 & 64.12\% & 57.23\%  \\ 
\hline
$16 \times 32$ & 50\% & 127 & 59.87\% & 54.26\%  \\ 
\hline
\end{tabular}
\caption{Results from pilot testing on the sparsity of negative feedback connections in a $24 \times 32$ Sparse Combo Net and a $16 \times 32$ Sparse Combo Net. Feedback Density refers to the percentage of possible subnetwork pairings that were trained in negative feedback, while the remaining inter-network connections were held at 0. All networks were trained with the same 150 epoch training paradigm as mentioned in the main text, but were stopped after hitting a 24 hour runtime limit. Decreasing Feedback Density is a promising path towards further improving performance as the size of Sparse Combo Nets is scaled. The ideal amount of feedback density will likely vary with the size of the combination network.}
\label{table:scalability}
\end{table}

\subsection{Architecture Interpretability}
In this subsection, we give additional information on trainable parameters and properties of the network weights, to assist in interpreting the Sparse Combo Net architecture and its training process. 

\subsubsection{Number of Parameters}\label{Appendix:num-params}
To report on the number of trainable parameters, we used the following formula:

$\frac{n^{2} - M*C^{2}}{2} + i * n + n * o + n + o$

Where $n$ is the total number of units in the $M \times C$ combination network, $o$ is the total number of output nodes for the task, and $i$ is the total number of input nodes for the task. Thus for the $16 \times 32$ networks highlighted here, we have 129034 trainable parameters for the MNIST tasks, and 130058 trainable parameters for sequential CIFAR10.

Note that the naive estimate for the number of trainable parameters would be $n^{2} + i * n + n * o + n + o$, corresponding to the number of weights in $\mathbf{L}$, the number of weights in the feedforward linear input layer, the number of weights in the feedforward linear output layer, and the bias terms for the input and output layers, respectively. However, because of the combination property constraints on $\mathbf{L}$, only the lower triangular portion of a block off-diagonal matrix is actually trained, and $\mathbf{L}$ is then defined in terms of this matrix and the metric $\mathbf{M}$. Thus we subtract $M*C^{2}$ to remove the block diagonal portions corresponding to nonlinear RNN components, and then divide by 2 to obtain only the lower half. 

\subsubsection{Inspecting Trained Network Weights}\label{Appendix:example-weights}
After training was completed, we inspected the state of all networks described in the main text, pulling both the nonlinear ($\mathbf{W}$) and linear ($\mathbf{L}$) weight matrices from both initialization time and the final model. For $\mathbf{W}$, we confirmed it did not change over training, and inspected the max real part of the eigenvalues of $|\mathbf{W}|$ in accordance with Theorem 1. The densest tested matrices tended to have $\lambda_{max}(|\mathbf{W}|) > 0.9$, while the sparsest ones tended to have $\lambda_{max}(|\mathbf{W}|) < 0.1$. For $\mathbf{L}$, we checked the maximum element and the maximum singular value before and after training. In general, both went up over the course of training, but by a modest amount. 

\begin{SCfigure}[0.5][h]
\centering
\caption{Weight matrices for each of the 32 unit nonlinear component RNNs that were used in the best performing $16 \times 32$ network on permuted sequential MNIST.}
\includegraphics[width=0.6\textwidth,keepaspectratio]{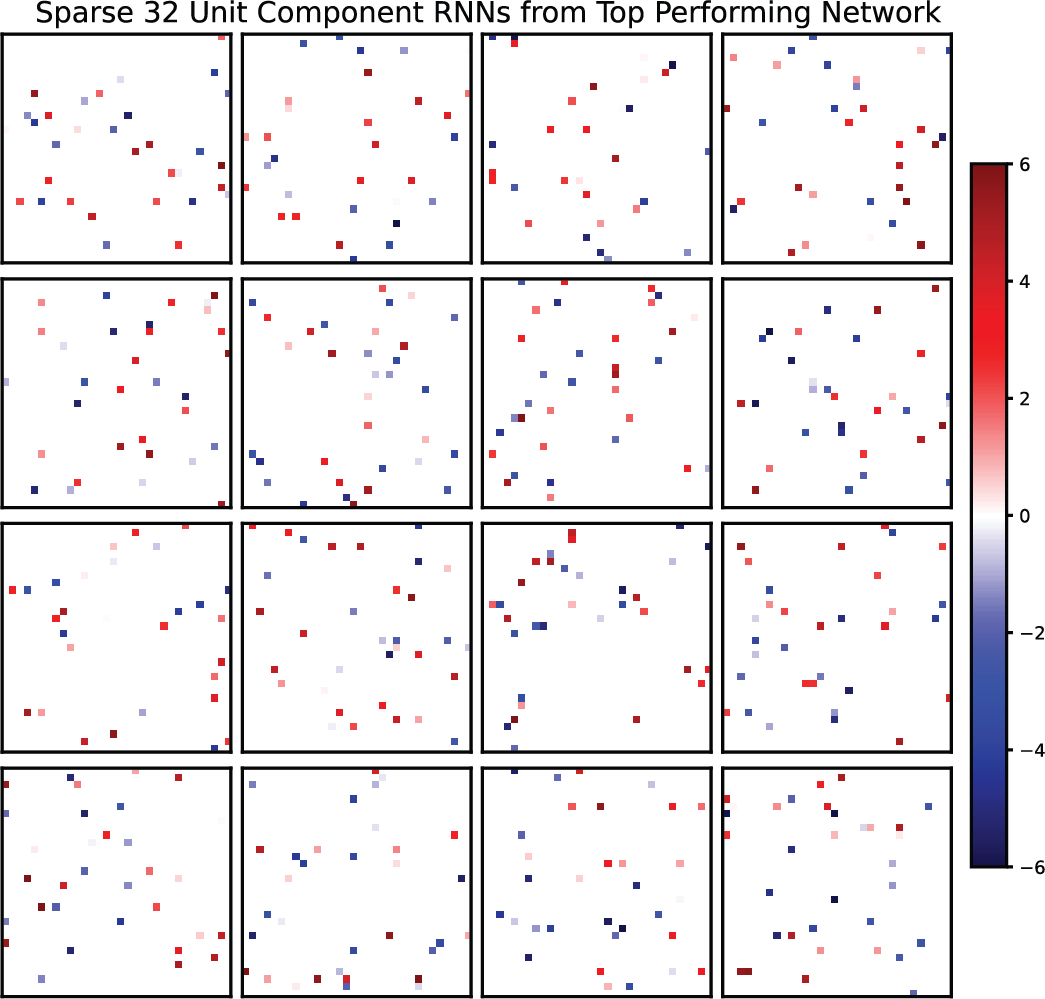}
\label{figure:example-best-net}
\end{SCfigure}

\subsection{Tables of Results by Trial (Supplementary Documentation)}\label{Appendix:all-tables}
Table \ref{table:hyperparams} shows all trials run on permuted sequential MNIST before beginning the more systematic experiments reported on in the main text. Notably, our networks did not require an extensive hyperparameter tuning process. 

Tables \ref{table:sizes-compare} and \ref{table:modules-compare} report additional details on the initial size and modularity experiments (Figure \ref{figure:test-sizes-sup}). Table \ref{table:compare-sizes-again} reports the results of the repeated experiment on Sparse Combo Net size, this time using 32 unit component subnetworks with best sparsity settings (Figure \ref{figure:test-sizes}A).

Tables \ref{table:sparsity-compare} and \ref{table:best-results} report results from all trials related to our sparsity experiments (Figures \ref{figure:test-sparsity} and \ref{figure:test-sparsity-sup}). 

Table \ref{table:psmnist-repeats} provides further information on the four trials in the permuted seqMNIST repeatability experiments (Figure \ref{figure:test-reproduce}). 

Table \ref{table:cifar} reports the results of all trials of different hyperparameters on the sequential CIFAR10 task, in chronological order. Ultimately the same settings as those used for permuted seqMNIST were chosen.

Table \ref{table:cifar-repeats} shows the results of all ten trials in the seqCIFAR10 repeatability experiments (Figure \ref{figure:cifar-reps}). Table \ref{table:cifar-repeats-short} shows results from nine additional seqCIFAR10 trials of shorter training duration (Figure \ref{figure:cifar-reps-shorter}), run to increase sample size while unable to access the higher quality Colab GPUs.  

Finally, Table \ref{table:scalability} reports the results of our pilot trial on introducing sparsity into the linear feedback connection matrix - as this table was presented in Section \ref{Appendix:scalability-discuss} however, we do not reproduce it here.

\newpage

\subsubsection{Permuted seqMNIST Trials}

\begin{table}[!htbp]
\centering
\begin{tabular}{ | m{1cm} | m{1cm} | m{1cm} | m{1cm} | m{4cm} | m{1cm} | }
\hline
 Size & Epochs & Adam WD & Initial LR & LR Schedule & Final Test Acc. \\
\hline\hline
$10 \times 16$ & 150 & 5e-5 & 5e-3 & 0.1 after 91 & 84\% \\
\hline
$10 \times 16$ & 150 & 1e-5 & 1e-2 & 0.1 after 50,100 & 85\% \\
\hline
$15 \times 16$ & 150 & 2e-4 & 5e-3 & 0.1 after 50,100 & 84\% \\ 
\hline
$10 \times 16$ & 150 & 2e-4 & 1e-2 & 0.5 every 10 & 81\% \\ 
\hline
$10 \times 16$ & 200 & 2e-4 & 1e-2 & 0.5 after 10 then every 30 & 81\% \\ 
\hline
$10 \times 16$ & 171* & 5e-5 & 1e-2 & 0.75 after 10,20,60,100 then every 15 & 84\% \\ 
\hline
\rowcolor{LightCyan}
$15 \times 16$ & 179* & 1e-5 & 1e-3 & 0.1 after 100,150 & 90\% \\ 
\hline
\end{tabular}
\caption{Training hyperparameter tuning trials, presented in chronological order. * indicates that training was cut short by the 24 hour Colab runtime limit. LR Schedule describes the scalar the learning rate was multiplied by, and at what epochs. The best performing network is highlighted, and represents the training settings we used throughout most of the main text.}
\label{table:hyperparams}
\end{table}

\begin{table}[!htbp]
\parbox{.45\linewidth}{
\centering
\begin{tabular}{ | m{1cm} || m{1cm} | m{1cm} | m{1cm} | }
\hline
 Size & Final Test Acc. & Epoch 1 Test Acc. & Final Train Loss \\
\hline\hline
$1 \times 16$ & 38.69\% & 24.61\% & 1.7005 \\  
\hline
$3 \times 16$ & 70.56\% & 40.47\% & 0.9033 \\
\hline
$5 \times 16$ & 77.86\% & 47.99\% & 0.7104 \\  
\hline
$10 \times 16$ & 85.82\% & 61.38\% & 0.4736 \\
\hline
$15 \times 16$ & 90.28\% & 69.09\% & 0.3156 \\  
\hline
$20 \times 16$ & 92.26\% & 71.72\% & 0.2392 \\
\hline
\rowcolor{LightCyan}
$22 \times 16$ & 93.01\% & 70.11\% & 0.2073 \\
\hline
$25 \times 16$ & 92.99\% & 61.81\% & 0.2017 \\  
\hline
$30 \times 16$ & 93.16\% & 43.21\% & 0.1991 \\
\hline
\end{tabular}
\caption{Results for combination networks containing different numbers of component 16-unit RNNs. Training hyperparameters and network initialization settings were kept the same across all trials, and all trials completed the full 150 epochs.}
\label{table:sizes-compare}}
\hfill
\parbox{.45\linewidth}{
\centering
\begin{tabular}{ | m{1cm} || m{1cm} | m{1cm} | m{1cm} | }
\hline
 Size & Final Test Acc. & Epoch 1 Test Acc. & Final Train Loss \\
\hline\hline
 $1 \times 352$ & 40.17\% & 26.97\% & 1.662 \\  
\hline
 $11 \times 32$ & 89.12\% & 61.29\% & 0.3781 \\
\hline
\rowcolor{Gray}
 $22 \times 16$ & 93.01\% & 70.11\% & 0.2073 \\ 
\hline
\rowcolor{LightCyan}
$44 \times 8$ & 94.44\% & 25.78\% & 0.1500 \\
\hline
$88 \times 4$ & 10.99\% & 10.99\% & 2E+35 \\  
\hline
\end{tabular}
\caption{Results for different distributions of 352 total units across a combination network. This number was chosen based on prior $22 \times 16$ network performance. For each component RNN size tested, the same procedure was used to select appropriate density and scalar settings. All networks otherwise used the same settings, as in the size experiments.}
\label{table:modules-compare}}
\end{table}

\begin{table}[!htbp]
\centering
\begin{tabular}{ | m{3cm} | m{1cm} || m{1cm} | }
\hline
Name & Size & Final Test Acc. \\
\hline\hline
Sparse Combo Net & $1 \times 32$ & 37.1\% \\  
\hline
Sparse Combo Net & $4 \times 32$ & 89.1\% \\ 
\hline
Sparse Combo Net & $8 \times 32$ & 93.6\% \\ 
\hline
Sparse Combo Net & $12 \times 32$ & 94.4\% \\ 
\hline
Sparse Combo Net & $16 \times 32$ & 96\% \\ 
\hline
\rowcolor{LightCyan}
Sparse Combo Net & $22 \times 32$ & 96.8\% \\ 
\hline
Sparse Combo Net & $24 \times 32$ & 96.7\% \\ 
\hline\hline
Control & $24 \times 32$ & 47\% \\ 
\hline
\end{tabular}
\caption{Results for Sparse Combo Nets containing different numbers of component 32-unit RNNs with best found initialization settings, using the standard 150 epoch training paradigm. This experiment was run to demonstrate repeatability of the size results seen in Table \ref{table:sizes-compare}. All trials were run to completion. A control trial was also run with the largest tested network size - the connections between subnetworks were no longer constrained, and thus this control combination network is not certifiably stable.}
\label{table:compare-sizes-again}
\end{table}

\begin{table}[!htbp]
\centering
\begin{tabular}{ | m{1cm} | m{1cm} | m{1cm} || m{1cm} | m{1cm} | m{1cm} | }
\hline
 Size & Density & Scalar & Final Test Acc. & Epoch 1 Test Acc. & Final Train Loss \\
\hline\hline
\rowcolor{Gray}
$11 \times 32$ & 26.5\% & 0.27 & 89.12\% & 61.29\% & 0.3781 \\  
\hline
$11 \times 32$ & 10\% & 1.0 & 94.86\% & 70.67\% & 0.1278 \\
\hline\hline
\rowcolor{Gray}
$22 \times 16$ & 40\% & 0.4 & 93.01\% & 70.11\% & 0.2073 \\  
\hline
\rowcolor{LightCyan}
$22 \times 16$ & 20\% & 1.0 & 95.27\% & 76.58\% & 0.0924 \\  
\hline
$22 \times 16$ & 10\% & 1.0 & 94.26\% & 71.53\% & 0.1425 \\  
\hline\hline
\rowcolor{Gray}
$44 \times 8$ & 60\% & 0.7 & 94.44\% & 25.78\% & 0.1500 \\
\hline
$44 \times 8$ & 50\% & 1.0 & 95.05\% & 30.52\% & 0.1180 \\  
\hline
\end{tabular}
\caption{Results for different initialization settings - varying sparsity and magnitude of the component RNNs for different network sizes. All other settings remained constant across trials, using our selected 150 epoch training paradigm.}
\label{table:sparsity-compare}
\vspace{0.5cm}
\begin{tabular}{ | m{1cm} | m{1cm} | m{1cm} || m{1cm} | m{1cm} | m{1cm} | }
\hline
 Density & Pre-select Scalar & Post-select Scalar & Final Test Acc. & Epoch 1 Test Acc. & Final Train Loss \\
\hline\hline
10\% & 1.0 & 1.0 & 95.87\% & 73.67\% & 0.074 \\  
\hline\hline
5\% & 10.0 & 0.1 & 95.11\% & 73.10\% & 0.1311 \\
\hline
5\% & 10.0 & 0.2 & 96.15\% & 82.50\% & 0.0051 \\
\hline
\rowcolor{LightCyan}
5\% & 10.0 & 0.5 & 96.69\% & 75.76\% & 0.0001 \\
\hline
5\% &  6.0 & 1.0 & 96.41\% & 21.55\% & 3.3E-5 \\
\hline
5\% & 7.5 & 1.0 & 16.75\% & 11.39\% & 3068967 \\
\hline\hline
3.3\% & 30.0 & 0.1 & 96.24\% & 83.89\% & 0.0005 \\
\hline
\rowcolor{LightCyan}
3.3\% & 30.0 & 0.2 & 96.54\% & 86.79\% & 4E-5 \\
\hline\hline
1\% & 10.0 & 1.0 & 96.04\% & 81.2\% & 0.0001 \\
\hline
\end{tabular}
\caption{Further optimizing the sparsity settings for high performance using a $16 \times 32$ network. The final scalar is the product of the pre-selection and post-selection scalars. Note that the 5\% density and 7.5 scalar network was killed after 18 epochs due to exploding gradient. All other trials ran for a full 150 epochs.}
\label{table:best-results}
\end{table}

\begin{table}[!htbp]
\centering
\begin{tabular}{ | m{1cm} | m{1.5cm} | m{1.5cm} | m{1.5cm} | }
\hline
Epochs & Best Test Acc. & Epoch 1 Test Acc. & Best Train Loss \\
\hline\hline
208 & 96.65\% & 61.15\% & 0.00224 \\  
\hline
255 & 96.94\% & 84.19\% & 5e-5 \\  
\hline
250 & 96.88\% & 81.08\% & 5e-5 \\  
\hline
210 & 96.93\% & 67.19\% & 0.00069 \\  
\hline
\end{tabular}
\caption{Repeatability of the best network settings on permuted seqMNIST. Four trials of $16 \times 32$ networks with 3.3\% density and entries between -6 and 6, trained for a 24 hour period with a single learning rate cut (0.1 scalar) after epoch 200. All other training settings remained the same as our selected hyperparameters. Trials are presented in chronological order. The mean test accuracy achieved was 96.85\% with Variance 0.019.}
\label{table:psmnist-repeats}
\end{table}

\newpage

\subsubsection{seqCIFAR10 Trials}

\begin{table}[!htbp]
\centering
\begin{tabular}{ | m{1cm} | m{1cm} | m{1cm} | m{1cm} | m{1cm} | m{1cm} | m{2.5cm} | m{1cm} | }
\hline
Density & Pre-select Scalar & Post-select Scalar & Epochs & Adam WD & Initial LR & LR Schedule & Best Test Acc. \\
\hline\hline
3.3\% & 30 & 0.2 & 150 & 1e-5 & 1e-3 & 0.1 after 90,140 & 64.63\% \\  
\hline
3.3\% & 30 & 0.2 & 34* & 1e-5 & 5e-3 & 0.1 after 90,140 & 35.42\% \\
\hline
5\% & 6 & 1 & 150 & 1e-5 & 1e-3 & 0.1 after 90,140 & 60.9\% \\ 
\hline
5\% & 10 & 0.5 & 150 & 1e-5 & 1e-4 & 0.1 after 90,140 & 54.86\% \\ 
\hline
3.3\% & 30 & 0.2 & 150 & 1e-5 & 5e-4 & 0.1 after 90,140 & 61.83\% \\ 
\hline
3.3\% & 30 & 0.2 & 200 & 1e-6 & 2e-3 & 0.1 after 140,190 & 62.31\% \\ 
\hline
\rowcolor{LightCyan}
3.3\% & 30 & 0.2 & 186* & 1e-5 & 1e-3 & 0.1 after 140,190 & 64.75\% \\ 
\hline
3.3\% & 30 & 0.2 & 132* & 1e-6 & 1e-3 & 0.1 after 140,190 & 62.31\% \\ 
\hline
5\% & 10 & 0.5 & 195* & 1e-5 & 1e-3 & 0.1 after 140,190 & 64.68\% \\ 
\hline
\end{tabular}
\caption{Additional hyperparameter tuning for the seqCIFAR10 task, presented in chronological order. * indicates that training was cut short by the 24 hour Colab runtime limit, or in the case of high learning rate was killed intentionally due to exploding gradient. LR Schedule describes the scalar the learning rate was multiplied by, and at what epochs. The best performing network is highlighted. Ultimately we decided on the same network settings and training hyperparameters for further testing, just extending the training period to 200 epochs with the learning rate cuts occurring after epochs 90 and 140.}
\label{table:cifar}
\end{table}

\begin{table}[!htbp]
\parbox{.475\linewidth}{
\centering
\begin{tabular}{ | m{1cm} | m{1cm} | m{1cm} | m{1cm} | m{1cm} | }
\hline
Epochs & Best Test Acc. & Epoch 1 Test Acc. & Best Train Loss \\
\hline\hline
186 & 64.75\% & 10.53\% & 0.83 \\  
\hline
200 & 64.04\% & 26.35\% & 0.787 \\ 
\hline
200 & 64.32\% & 26.76\% & 0.778 \\ 
\hline
170 & 64.88\% & 20.65\% & 0.857 \\ 
\hline
200 & 65.72\% & 27.28\% & 0.813 \\ 
\hline
200 & 63.73\% & 10.47\% & 0.826 \\ 
\hline
200 & 65.03\% & 18.75\% & 0.83 \\ 
\hline
200 & 64.71\% & 32.36\% & 0.799 \\ 
\hline
200 & 64.4\% & 10.09\% & 0.83 \\ 
\hline
200 & 65.63\% & 30.25\% & 0.792 \\ 
\hline
\end{tabular}
\caption{Repeatability of the best network settings on seqCIFAR10. Ten trials of $16 \times 32$ networks with 3.3\% density and entries between -6 and 6, trained for 200 epochs with learning rate scaled by 0.1 after epochs 140 and 190. All other training settings remained the same as before. Trials are presented in chronological order. The mean test accuracy achieved was 64.72\% with Variance 0.406. Most trials completed all 200 epochs, but two were cut short due to runtime limits.}
\label{table:cifar-repeats}}
\hfill
\parbox{.475\linewidth}{
\centering
\begin{tabular}{ | m{1cm} | m{1cm} | m{1cm} | m{1cm} | m{1cm} | }
\hline
Epochs & Best Test Acc. & Epoch 1 Test Acc. & Best Train Loss \\
\hline\hline
150 & 64.63\% & 28.91\% & 0.837 \\  
\hline
150 & 63.51\% & 9.7\% & 0.9 \\ 
\hline
148 & 62.35\% & 17.51\% & 0.89 \\ 
\hline
126 & 63.33\% & 23.61\% & 0.903 \\ 
\hline
129 & 61.45\% & 28.93\% & 0.937 \\ 
\hline
150 & 62.21\% & 25.82\% & 0.9 \\ 
\hline
150 & 63.42\% & 23.51\% & 0.86 \\ 
\hline
147 & 62.05\% & 27.67\% & 0.882 \\ 
\hline
131 & 62.41\% & 30.33\% & 0.912 \\ 
\hline
\end{tabular}
\caption{An additional nine trials investigating the repeatability of our results on the seqCIFAR10 task. For these trials we used the same 150 epoch training paradigm as previously, although only four of the networks were able to fully complete training. These trials were done to expand our sample size while access was limited to only slower GPUs. The mean observed test accuracy among the shorter trials was 62.82\%, with variance of 0.95.}
\label{table:cifar-repeats-short}}
\end{table}

\newpage

\section{SVD Combo Net Details}\label{Appendix:SVDNet}
\subsection{Parameterization Information}

For the SVD Combo Net, we ensured contraction by directly parameterizing each of the $\mathbf{W}_i$ ($i = 1,2 \dots p$) as:

\begin{equation}\label{eq: W_param}
\begin{split}
\mathbf{W}_i =  \mathbf{\Phi}_i^{-1} \mathbf{U}_i \mathbf{\Sigma}_i \mathbf{V}_i^T \mathbf{\Phi}_i
\end{split}
\end{equation}

where $\mathbf{\Phi}_i$ is diagonal and nonsingular, $\mathbf{U}_i$ and $\mathbf{V}_i$ are orthogonal, and $\mathbf{\Sigma}_i$ is diagonal with $\Sigma_{ii} \in [0,g^{-1})$. We ensure orthogonality of $\mathbf{U}_i$ and $\mathbf{V}_i$ during training by exploiting the fact that the matrix exponential of a skew-symmetric matrix is orthogonal, as was done in \citep{lezcano2019cheap}. The network constructed from these subnetworks using \eqref{eq:combo_RNN} is contracting in metric $\Tilde{\mathbf{M}} = \text{BlockDiag}(\mathbf{\Phi}^2_1, \dots, \mathbf{\Phi}^2_p)$.

\subsection{Control Experiment}

To do the SVD Combo Network stability control experiment, we still constrained the weights of the subnetwork modules using the above parameterization, but we removed any constraint on the connections between modules. Thus each individual module was still guaranteed to be contracting, but the overall system had no such guarantee in this control trial. The experiment was run using a $24 \times 32$ combination network and the permuted sequential MNIST task.

In contrast to the primary control experiment run on Sparse Combo Net, for the SVD Combo Net we saw only a very slight performance decrease when the between-module connections were no longer constrained for the overall stability certificate. Test accuracy in the control run was $94.56\%$, as compared to the $94.9\%$ observed with the standard SVD Combo Net. 

This disparity in performance decrease makes sense when considering that the hidden-to-hidden weights of SVD Combo Network are trainable, while those of the Sparse Combo Network are not. Whatever instabilities are introduced by the lack of constraint on the inter-subnetwork connections likely cannot be adequately compensated for in the Sparse Combo Network. 

\subsection{Model Code}
\label{Appendix:Model}

\begin{figure}[h]\label{fig: SVD_Combo_Net_cell}
\centering
\includegraphics[width =\textwidth,keepaspectratio]{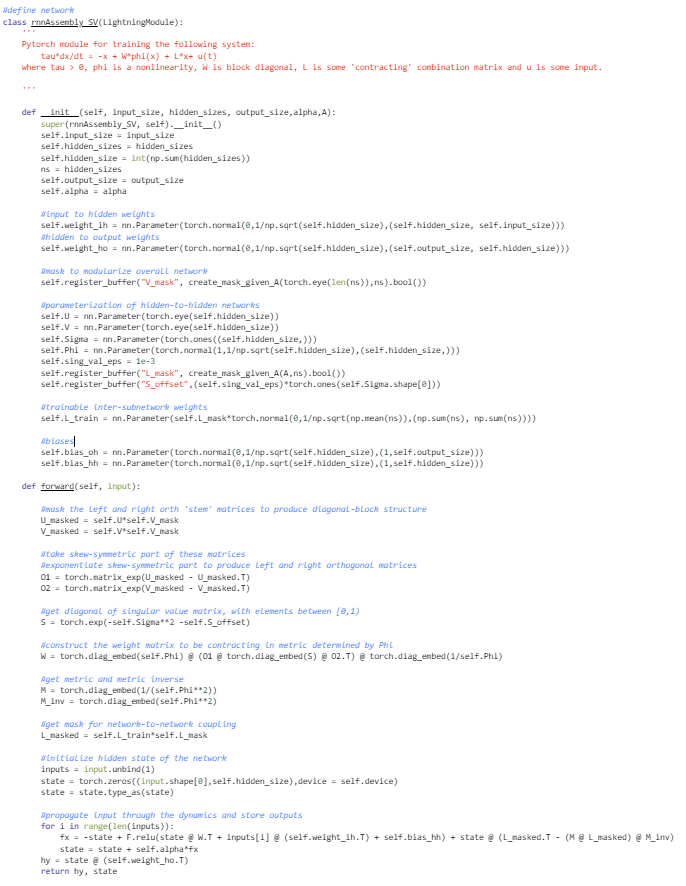}
\caption{Pytorch Lightning code for SVD Combo Net cell.} 
\end{figure}

\end{document}